\documentclass{article} 
\usepackage[preprint]{goodfire}

\usepackage{microtype}
\usepackage{hyperref}
\usepackage{url}
\usepackage{booktabs}
\usepackage{amsmath}
\usepackage{macros}
\usepackage{lineno}
\usepackage[color=pink!20,textsize=scriptsize]{todonotes}

\hypersetup{colorlinks=true, citecolor=primary, linkcolor=primary, urlcolor=primary}

\providecommand{\aut}[1]{\textbf{#1}}
\providecommand{\af}[1]{{\small #1}}

\providecommand{\afn}[1]{\textcolor{primary}{$^{#1}$}}

\newcommand{\costar}{\textcolor{secondary}{\boldsymbol{\star}}}
\newcommand{\colead}{\textcolor{secondary}{\boldsymbol{\dagger}}}

\newcommand{\goodfiremark}{\raisebox{0.5pt}{\hspace{0.5mm}\includegraphics[height=6pt]{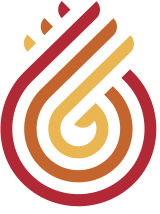}}}

\newcommand{\goodfireaff}{%
  \includegraphics[height=14pt]{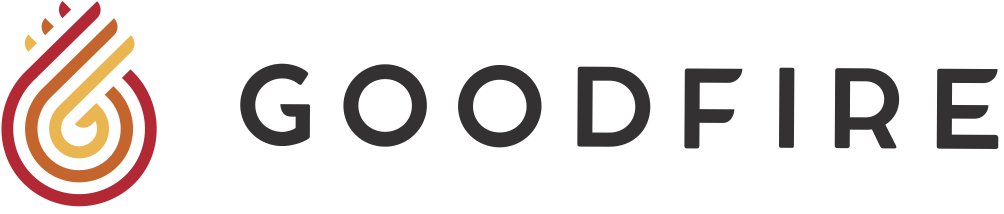}
}

\newcommand{\authorentry}[2]{\aut{#1}\afn{#2}}
\newcommand{\afflabel}[2]{\afn{#1}\af{#2}}
\newcommand{\authorsep}{\quad}
\newcommand{\repolink}[1]{%
  {\small
    \href{#1}{\raisebox{-2.8pt}{\includegraphics[height=10pt]{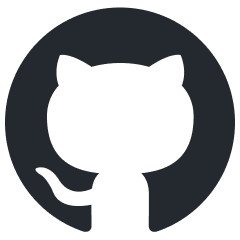}}%
    \texttt{\textcolor{secondary}{#1}}}
  }%
}

\newtoggle{goodfireonly}
\togglefalse{goodfireonly}  
\newcommand{\extline}[1]{\iftoggle{goodfireonly}{}{#1}}
\newcommand{\extaff}[1]{\iftoggle{goodfireonly}{}{#1}}

\newcommand{\paperauthors}{%
\authorentry{Usha Bhalla}{\costar\goodfiremark\extaff{,a}} \authorsep
\authorentry{Thomas Fel}{\costar\goodfiremark} \authorsep
\authorentry{Can Rager}{\goodfiremark} \\
\vspace{2pt}
\authorentry{Sheridan Feucht}{\goodfiremark\extaff{,b}} \authorsep
\authorentry{Tal Haklay}{\goodfiremark\extaff{,c}} \authorsep
\authorentry{Daniel Wurgaft}{\goodfiremark\extaff{,d}} \authorsep
\authorentry{Siddharth Boppana}{\goodfiremark} \\
\vspace{2pt}
\authorentry{Matthew Kowal}{\goodfiremark} \authorsep
\authorentry{Vasudev Shyam}{\goodfiremark} \authorsep
\authorentry{Owen Lewis}{\goodfiremark} \authorsep
\authorentry{Thomas McGrath}{\goodfiremark} \authorsep
\\
\vspace{2pt}
\authorentry{Jack Merullo}{\goodfiremark} \authorsep
\authorentry{Atticus Geiger}{\colead\goodfiremark} \authorsep
\authorentry{Ekdeep Singh Lubana}{\colead\goodfiremark}
\vspace{4pt}\\
\textcolor{secondary}{$^{\boldsymbol{\star}}$}\af{Equal contribution} \authorsep
\vspace{2mm}
\textcolor{secondary}{$^{\boldsymbol{\dagger}}$}\af{Equal senior contribution} \\
\goodfireaff \\
\extline{%
  \afflabel{a}{Harvard University} \authorsep
  \afflabel{b}{Northeastern University} \authorsep
  \afflabel{c}{Technion IIT} \authorsep
  \afflabel{d}{Stanford University}
}
\vspace{3mm}\\
\repolink{~https://github.com/goodfire-ai/sae-manifold}
\vspace{-4mm}
}
\author{\paperauthors}

\title{Do Sparse Autoencoders Capture Concept Manifolds?
}

\begin{document}


\maketitle

\vspace{-4mm}
\begin{abstract}
Sparse autoencoders (SAEs) are widely used to extract interpretable features from neural network representations, often under the implicit assumption that concepts correspond to independent linear directions. 
However, a growing body of evidence suggests that many concepts are instead organized along low-dimensional manifolds encoding continuous geometric relationships. 
This raises three basic questions: what does it mean for an SAE to capture a manifold, when do existing SAE architectures do so, and how?
We develop a theoretical framework that answers these questions and show that SAEs can capture manifolds in two fundamentally different ways: \emph{globally}, by allocating a 
compact group of atoms whose linear span contains the entire manifold, or \emph{locally}, by distributing it across features that each selectively \textit{tile} a restricted region of the underlying geometry.
Empirically, we find that SAEs suboptimally recover continuous structures, 
mixing the global subspace and local tiling solutions in a fragmented regime we call \emph{dilution}.
This explains why manifold structure is rarely visible at the level of individual concepts and motivates post-hoc unsupervised discovery methods that search for coherent groups of atoms rather than isolated directions.
More broadly, our results suggest that future representation learning methods should treat geometric objects, not just individual directions, as the basic units of interpretability.
\end{abstract}

\section{Introduction}

\begin{wrapfigure}{r}{0.47\textwidth}
    \vspace{-8mm}
    \centering
    \includegraphics[width=0.97\linewidth]{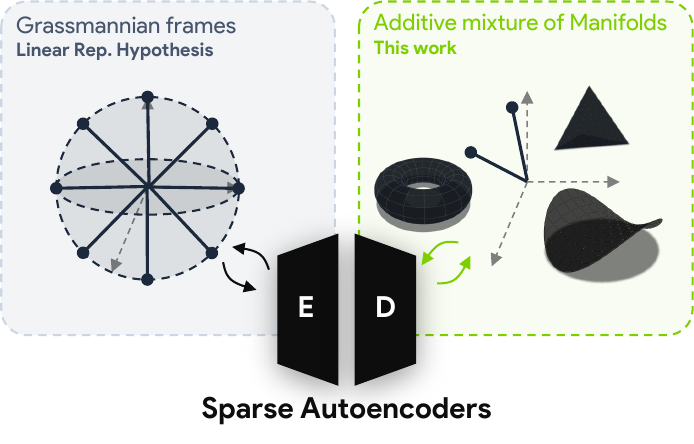}
    \vspace{-4mm}
    \caption{\small{\textbf{From directions to manifolds.} Under the linear representation hypothesis, concepts correspond to individual directions in activation space, packed as a Grassmannian frame~\citep{strohmer2003grassmannian}. We consider the richer setting where concepts are organized along low-dimensional manifolds that are additively superposed and ask whether and how SAEs recover these geometric objects.
    \vspace{-3mm}}}
    \label{fig:intro}
\end{wrapfigure}

Motivated by unprecedented improvements in Large Language Models' (LLMs) capabilities, recent work has sought to understand why an LLM produces a particular output for a given input~\citep{sharkey2025open}.
Often, such work makes assumptions about the geometry of neural network representations, arguing especially that abstract concepts (latent factors) underlying the data-generating process are represented in a ``linear'' fashion~\citep{elhage2022superposition, olahreps, arora2018linear, jiang2024origins}.
Called the \textit{Linear Representation Hypothesis (LRH)}~\citep{park2023linear, zheng2025model}, this geometric model argues a neural network's representations are an additive mixture of several directions, each encoding a specific concept~\citep{elhage2022superposition};
any concept's value can be read-out from a neural network's hidden representations via a linear map~\citep{belinkov2022probing}; 
and linear algebraic operations suffice to manipulate this value~\citep{mikolov2013efficient, korchinski2025emergence, karkada2025closed}.
LRH can thus be deemed as a generative model of neural network representations, the inverse of which leads to Sparse Autoencoders (SAEs)~\citep{costa2025flat}---a popular tool used for unsupervised discovery of concepts learned by a model~\citep{bricken2023monosemanticity, gao2024scaling, bussmann2024batchtopk, rajamanoharan2024jumping, bussmann2025learning}, with deep roots in the older literature on sparse coding~\citep{olshausen1996emergence,olshausen1997sparse,klindt2020towards,klindt2025superposition}, sparse subspace clustering~\citep{elhamifar2013sparse,abdolali2021beyond}, and nonlinear manifold learning~\citep{tenenbaum2000global,roweis2000nonlinear}.

\begin{figure}
   \centering
   \vspace{-10pt}
   \includegraphics[width=0.9\linewidth]{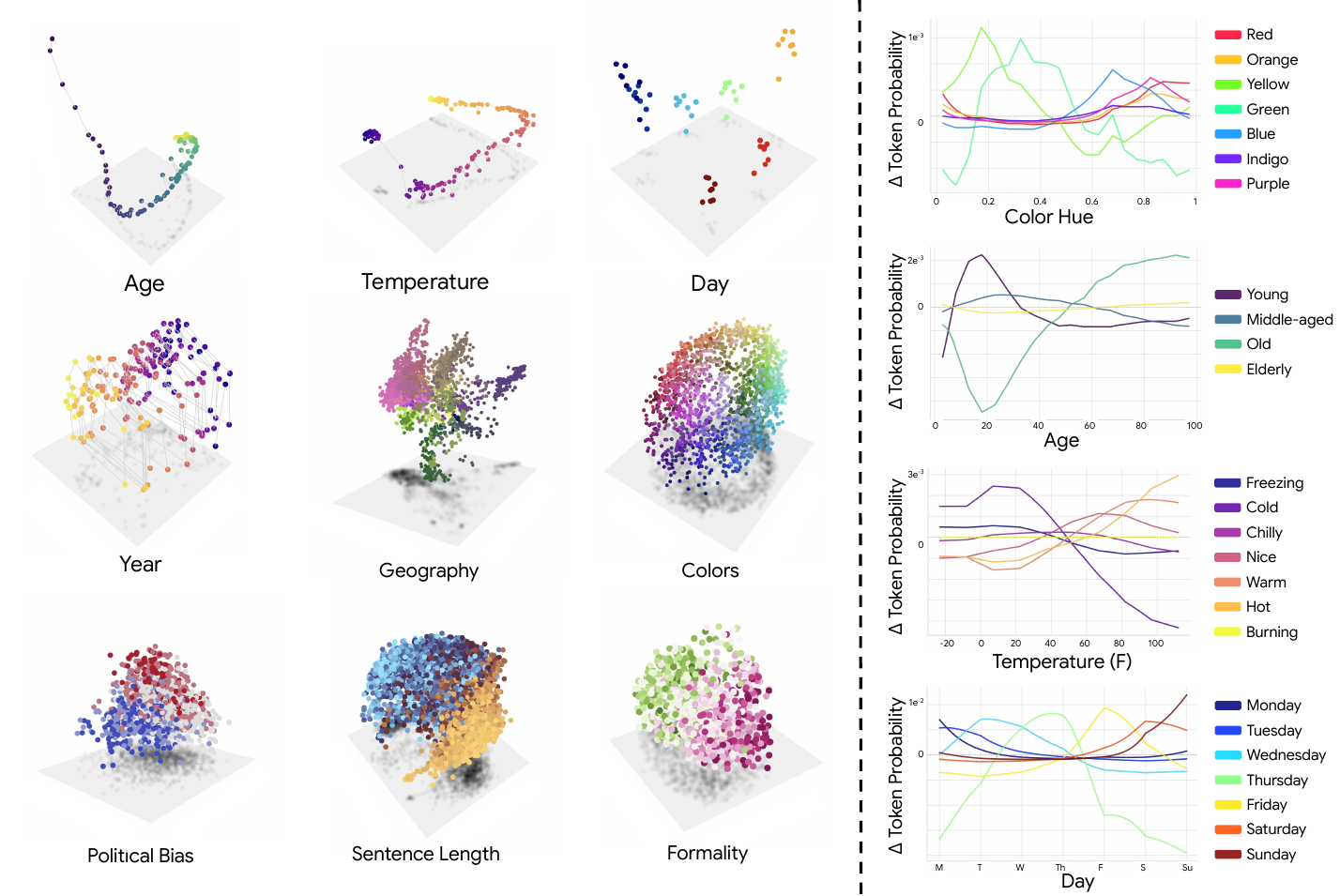}
   \caption{\small{\textbf{Evidence of manifold structure in model representations and its effect on behavior.} (Left) PCA projections of Llama3.1-8B layer 19 activations corresponding to continuous concepts (e.g., age, temperature, day, color) reveal smooth geometric structure rather than isolated directions.
   (Right) Steering interventions between concept centroids (e.g., Wednesday to Thursday) produce smooth changes in token probabilities for concept-dependent outputs.} 
   \vspace{-15pt}
    }
   \label{fig:manis}
\end{figure}

While SAEs have been used at scale to some success, e.g., to debug neural networks deployed at scale~\citep{OAISAE, nguyen2025deploying} or to identify candidate biomarkers learned by an epigenetics model~\citep{wang2026alzheimers}, a growing body of recent work has argued that geometry of neural network representations is more intricate than LRH suggests~\citep{lubana2025priors,fel2025into,karkada2026symmetry,dooms2025finding}: e.g., periodic concepts are encoded along a circular topology~\citep{modell2025, kantamneni2025language, engels2024not}; open-ended numerical concepts along a linear topology~\citep{gurnee2025when, yocum2025neural}; in-context statistics induce arbitrary graph structured representations~\citep{park2025iclrincontextlearningrepresentations, saanum2025circuit, sarfati2026shape}; hierarchical representations are seen in genomics~\citep{pearce2025tree} and vision-language models~\citep{costa2025flat}; syntactic relations are organized along a polar coordinate system that jointly encodes the existence and the type of dependencies~\citep{diego2024polar}; and spatially and temporally smooth representations emerge in vision and language models, respectively~\citep{chung2018classification, cohen2020separability,  lubana2025priors, hosseini2026context,dhimoila2026cross,fel2025archetypal,gorton2024missing}.
\cite{karkada2026symmetry}'s results in fact show that these representation geometries reflect uncertainty across different values a concept can take, hence endowing meaning to distances between two points in the representation space---a property directly at odds with LRH, which primarily focuses on directions.
These results then motivate the question: if representations of a concept exhibit structure outside the scope of LRH, \textbf{do SAEs capture such manifolds\footnote{A note on the word: ``manifold'' is partly convention~\citep{chung2018classification, cohen2020separability, pearce2025tree, modell2025} and partly hope. Empirically, we mean curved, low-dimensional structures that representations appear to lie on; we adopt the strict differential-geometric definition as a working assumption. Whether real representations satisfy that assumption, and whether ``manifold'' survives as the right name, remains an open problem.}?}
In particular, the mismatch between the assumptions made by SAEs and the underlying geometry of model activations does not by itself reject SAEs as valuable interpretations of model representations. If SAEs perform reconstruction well, then their activations must necessarily preserve the geometry of model representations. 
The key issue is therefore not whether the geometry is preserved, but whether it is organized in a useful and interpretable way.
To address this question, we make the following contributions. 

\begin{itemize}[leftmargin=*]
\item \textbf{Formalizing the Problem of Capturing Manifolds using SAEs.} We first demonstrate that a plethora of manifolds, i.e., nonlinearly curved geometric structures with causal efficacy, exist in representations of a pretrained LLM (Sec.~\ref{sec:qualitative_proof}).
Inspired by prior work in neuroscience~\citep{khona2022attractor, eichenbaum2018barlow}, we formalize the problem of capturing such manifolds via sparse coding and show that if features (rows of an SAE decoder) specialize to specific values of a concept, such that different features cover different values, then the SAE can still satisfy its architectural constraints (e.g., sparsity), achieve good reconstruction, and yet capture the curved geometries underlying neural network representations by ``tiling'' the manifold with its features (Sec.~\ref{sec:formalization}). 
Interestingly, these results also yield an impossibility claim for current SAE architectures directly motivated by manifold learning algorithms.

\item \textbf{Exhaustively Characterizing How SAEs Tile Manifolds.} Moving beyond the theoretical possibility of SAEs tiling manifolds, we perform a thorough characterization to demonstrate this mechanism manifests in both natural settings and synthetic datasets.  
We showcase ``tuning curves''~\citep{butts2006tuning}, highlighting the selectivity of SAE features for specific values of a concept: splitting into finer-than-necessary grained buckets \citep{lubana2025priors, bricken2023monosemanticity, chanin2024absorption}, while other parts are represented in a redundant manner across several features.
This suggests that low reconstruction error alone does not guarantee coherent manifold recovery. 
In practice, SAEs often represent manifolds through fragmented collections of atoms that behave like localized detectors, rather than a coherent global structure.

\item \textbf{Unsupervised Discovery of Manifold Structures.} 
Toward a predictive account, we define an optimization problem motivated by the classical Ising model in Physics~\citep{schneidman2006weak} to identify features whose co-activation statistics are either strongly correlated or anti-correlated.
This unsupervised method helps identify both manifolds tiled by SAE features that we could find via supervised data (Fig.~\ref{fig:manis}) and novel ones.
Critically, our results show that mere correlation of feature directions, as used in prior work~\citep{engels2024not}, need not suffice to find manifolds.
\end{itemize}

More broadly, the perspective put forward in this paper suggests structured, nonlinear geometries are ubiquitous in model representations and are likely to be the unit of computation in which frameworks of interpretability should be defined.
To this end, we either need protocols that actively seek to isolate manifolds from a model's representations or perform posthoc analysis of SAE features to identify such geometries.
Despite positive results, we note mixed selectivity features make the latter an unreliable option, but until novel featurizers are developed, are our current best option.

\vspace{-15pt}

\begin{goodfirenote}{A Geometric Note on Prior Negative Results for SAEs.}
This perspective has direct practical value in that it explains a number of limitations and phenomena previously documented for SAEs. 

~~\i{i}~\textbf{Instability of learned dictionaries.} The apparent instability of learned features across runs can arise naturally when concepts are organized along manifolds: multiple equivalent bases or tilings can represent the same underlying geometry, so different initializations may yield different but equally valid decompositions.

~~\i{ii}~\textbf{Steering through individual features is often ineffective or brittle.} If features correspond to local directions or patches of a curved manifold, then intervening on a single feature pushes representations off the manifold rather than moving coherently along it

~~\i{iii}~\textbf{Difficulty of automated interpretability.} Only taken together do features span a subspace whose structure makes sense of the underlying object; inspecting directions in isolation will fail to recover any meaning, or worse, hallucinate a spurious one.
\\

Together, these observations suggest that many negative findings about SAEs reflect a mismatch between direction-based interpretations and the nonlinear geometry of representations, rather than a failure of the models themselves.

\end{goodfirenote}

\section{Notations: Sparse Coding and SAEs}
Throughout, vectors are denoted by lowercase bold letters (e.g.,
$\bm{x}$) and matrices by uppercase bold letters (e.g., $\bm{X}$).
We use $[n]$ for the set $\{1, \dots, n\}$.
We write
$\mathcal{B}^{c \times d} = \{\bm{M} \in \mathbb{R}^{c \times d}
\mid \|\bm{M}_{i,:}\|_2 = 1,\; \forall i\}$ for the set of
matrices with unit-norm rows.
For a matrix $\bm{V} \in \mathbb{R}^{k \times d}$, we write
$\mathrm{Im}(\bm{V}) = \{\bm{x}\bm{V} : \bm{x} \in \mathbb{R}^k\}
\subseteq \mathbb{R}^d$ for its row span and $\bm{X} \geq \bm{0}$ (or $\bm{x} \geq \bm{0}$) indicates element-wise non-negativity.
It is well-established that current approaches for concept recovery from neural network representations are fundamentally instances of sparse coding~\citep{fel2023holistic,fel2025archetypal,hindupur2025projecting}. 
Briefly, sparse coding assumes a generative model where data points are produced by a sparse linear combination of latent variables (the concepts)~\citep{olshausen1996emergence, olshausen1997sparse}. 
Given an input $\bm{x}$, the goal is to extract its underlying generative factors using an overcomplete dictionary.  

\begin{definition}[Sparse Autoencoders]
\label{def:concept_duality}
Given an activation $\bm{x} \in \mathcal{A}$, SAEs extract a latent representation $\bm{z} \in \mathbb{R}^c$ via a dictionary $\bm{D} \in \mathbb{R}^{c \times d}$ by solving the following optimization:
%
\begin{equation}
\small
\label{eq:dictionarylearning}
\argmin_{\substack{\bm{W},~ \bm{D}\in\Omega}} \|\bm{x} - \bm{z}\bm{D}\|_2^2 
+ \lambda \mathcal{R}(\bm{z}) 
\quad \text{with} \quad \bm{z} = \operatorname{ReLU}( \bm{x}\bm{W}), \quad \Omega = \mathcal{B}^{c \times d}
\end{equation}
where $\mathcal{R}(\bm{z})$ is a sparsity-promoting regularizer (e.g., restricting $\|\bm{z}\|_0 \leq k$). Consequently, the localized reconstructions $\hat{\bm{x}} = \bm{z}\bm{D}$ lie in a sparse non-negative span (a cone).
\end{definition}

\section{Manifolds are Ubiquitous in LLM Representations}
\label{sec:qualitative_proof}

Before we proceed further with a detailed study of how SAEs capture curved geometries, i.e., manifolds, we first show these objects are a construct worth studying.
To this end, we build on recent results showing language model representations reflect symmetries in data statistics, resulting in curved representation geometries~\citep{karkada2026symmetry}.
Many real-world concepts in fact exhibit such inherent continuity and structure, taking values that smoothly vary along some range (e.g., temperature varies along the real line).
Such concepts can thus be expected to be represented along low-dimensional geometric objects embedded in a high-dimensional space. 
Building on this, we take several domains where a concept can be continuously varied, define a template in which a variable takes on values from this concept (see App.~\ref{app:manifold_details} for details), and sample several strings that vary primarily along this concept's value.
Performing a PCA of these representations then results in curved, often nonlinear geometries shown in Fig.~\ref{fig:manis} (left): this includes concepts characterized in prior work, e.g., days of the week organized in a cycle \citep{engels2024not}, and also new ones, e.g., colors organized along a paraboloid with circular hue and lightness dimensions, and spatial or temporal variables.
In all cases in Fig.~\ref{fig:manis} (left), we see distances and neighborhoods encode semantic similarity, i.e., nearby points correspond to similar meanings.

To emphasize these results further and highlight more strongly the disparity of these manifolds being outside the scope of LRH, we showcase that these manifolds are not merely geometric artifacts but are functionally relevant by measuring their effect on model behavior. 
Specifically, we find we can steer along the manifolds by steering between prototypical centroids (e.g., center of ``Wednesday" tokens) and smoothly interpolating between those points. 
For tasks that depend on the underlying variable—such as predicting color names from hex codes or describing temperature in natural language—we observe that model outputs change smoothly and predictably along these interpolations (Fig.~\ref{fig:manis}, right). 
This indicates that the manifold structure is not only present in the representation but also causally influences downstream behavior.

\section{Formalizing Manifold Capture in Sparse Representations}
\label{sec:formalization}

To concretize what it means to successfully ``capture'' manifolds identified from off-the-shelf pretrained LLMs using SAEs, we first analyze an abstraction that extends LRH to concepts with multi-dimensional, nonlinearly curved geometries.
We call this model of representations the ``Additive Mixture of Manifolds'' (see Figure \ref{fig:intro}).
We emphasize we are not making a normative claim here that neural networks satisfy this model of representations; instead, our goal is to take a concrete scenario where we can be rigorous, define useful metrics, and then see how results generalize to off-the-shelf models.

\paragraph{Representations as Additive Mixture of Manifolds.}
The Linear Representation Hypothesis (LRH) models each concept as a ray in activation space: a single direction scaled by a coefficient~\citep{park2023linear,costa2025flat}. 
The geometric structure identified in Sec.~\ref{sec:qualitative_proof} suggests that this is a special case of a richer phenomenon in which concepts vary continuously over low-dimensional surfaces, the formal description of which can be attributed to several prior works~\citep{modell2025,fel2025into,costa2025flat,lubana2025priors}.

\begin{definition}[Additive Mixture of Manifolds]
\label{def:mrh}
Let $\mathcal{M}_1, \ldots, \mathcal{M}_m \subset \mathbb{R}^d$ be
compact smooth submanifolds with ambient dimensionality 
$\dim(\mathcal{M}_i) = d_i \ll d$. Let $\bm{f}_{i}:\mathcal{M}_i\rightarrow \mathbb{R}^{d}$ be the immersion maps from each submanifold into $\mathbb{R}^d$.
Additive mixture of manifolds defines a model of representations wherein representations decompose into a superposition of manifolds as follows.
\begin{equation}
\small
\label{eq:superposition}
\bm{x} = \sum_{i \in S \subseteq [m]} \bm{f}_i(\bm{m}_i),
\qquad \bm{m}_i \in \mathcal{M}_i, \quad |S| \ll m.
\vspace{-5pt}
\end{equation}
\end{definition}

In other words, $\bm{x}$ lives in a Minkowski sum of the immersed submanifolds $\mathcal{M}_i$. 
When each $\mathcal{M}_i$ is a ray ($d_i = 1$), every term
$\bm{f}_i(\bm{m}_i)$ is a scalar multiple of a fixed direction, recovering the LRH.
In the general case, each manifold is contained in a
$k_i$-dimensional affine subspace and admits a parametrization
$\bm{m}_i(\bm{\theta}) =
\bm{\gamma}_i(\bm{\theta})\,\bm{V}_i + \bm{b}_i$. Here,
$\bm{\theta} \in \Theta_i \subseteq \mathbb{R}^{d_i}$ represents the
intrinsic coordinates, and the map
$\bm{\gamma}_i$ is a smooth embedding.
Furthermore, $\bm{V}_i \in \mathbb{R}^{k_i \times d}$ is an orthonormal
basis matrix, and $\bm{b}_i \in \mathrm{Im}(\bm{V}_i)$ is a translation
vector.
Importantly, superposition arises when $\sum_i k_i > d$.
Now that we have formally defined the target object of our interest, we are ready to examine what it mathematically means to capture a manifold.

\paragraph{Subspace Recovery via SAEs}

\begin{wrapfigure}[19]{r}{0.47\textwidth}
    \centering
    \vspace{-15pt}
    \includegraphics[width=\linewidth]{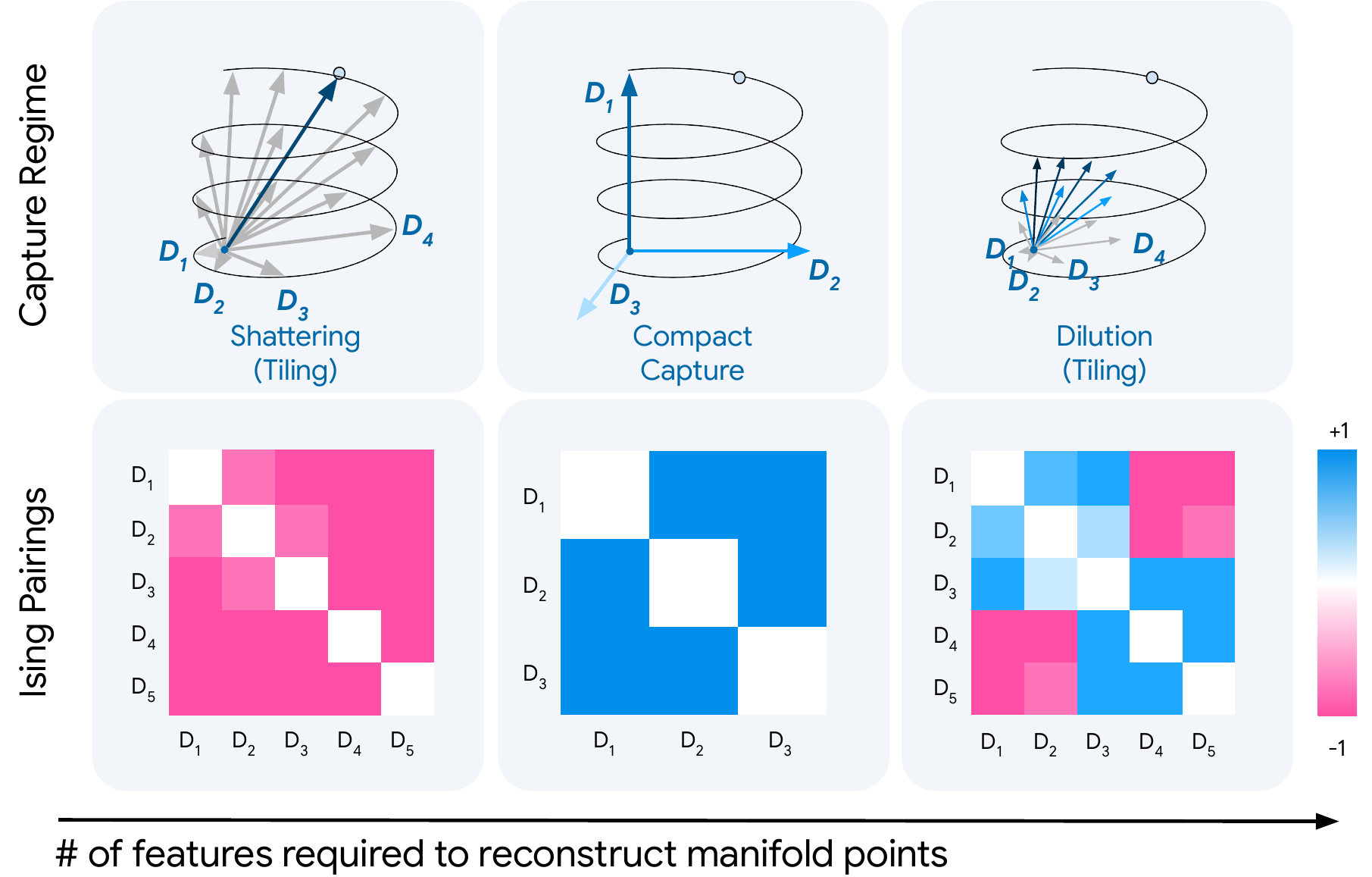}
    \caption{\small{\textbf{Tiling vs.\ Capture.} When features are highly selective, manifolds are ``tiled'' by shattering into sub-parts and features show anti-correlated occurrences (left). Compact capture involves features jointly reconstructing the manifold with no selectivity, resulting in positive couplings for the full support (middle). Dilution occurs when many redundant atoms activate to tile the manifold, but with feature sets of mixed selectivity (right).}
    }
    \label{fig:ising_hypo}
\end{wrapfigure}

It is easy to see that for an SAE to reconstruct a representation $\bm{x}$, $\bm{x}$ ought to lie in the linear span of its decoder.
The central observation we posit in this section is that an SAE captures a manifold well when a small, fixed group of atoms spans a subspace containing it, and the encoder consistently selects this group on every input from the manifold.

\begin{definition}[Subspace capture]\label{def:subspace_capture}
An SAE captures $\mathcal{M}$ at precision $\varepsilon$ if there exists 
$S^\star \subset [c]$ with $|S^\star| \leq k_{\mathcal{M}}$ such that
\begin{equation}
    \bigl\| \bm{x}_m - \sum_{i \in S^\star} z_i(\bm{x}_m)\, \bm{D}_i \bigr\| \;\leq\; \varepsilon
    \quad \forall\, \bm{x}_m \in \mathcal{M}.
\end{equation}
\end{definition}

Intuitively speaking, the definition says that a few decoder directions serve as a low-dimensional detector for $\M$. 
This is parsimonious (few atoms for the whole manifold) and coherent (the same atoms fire for every input on $\M$), and under an additional assumption
it is possible to establish a condition for the capture of a manifold in the sense of Defn.~\ref{def:subspace_capture}.
We note that this formulation is closely related to the subspace-preserving recovery condition that grounds sparse subspace clustering~\citep{elhamifar2013sparse, soltanolkotabi2014robust,tschannen2018noisy}; we provide a detailed treatment of this connection in App.~\ref{app:scc}.

\begin{theorem}[Subspace recovery]\label{thm:subspace_recovery}
Let $\M$ lie in a $k$-dimensional affine subspace with orthonormal basis 
$\bm{V}$. Let $\bm{D}$ be $\mu$-incoherent, and suppose there exists 
$\substar \subset [c]$ with $|\substar| = k$ such that 
$\mathrm{Im}(\bm{V}) = \mathrm{span}(\bm{D}_{\substar})$ and $\mu < 1/(2k-1)$. 
If the SAE achieves reconstruction error $\varepsilon(\M) \leq \lambda$, then 
an idealized sparse decoder over $\bm{D}$ captures $\M$ at precision $O(\lambda)$.
\end{theorem}

The proof relies on classical results in sparse dictionary learning~\citep{donoho2003optimally, tropp2004greedy, tropp2006just}; see App.~\ref{app:theory:directions} for details, including a discussion of the amortization gap between the idealized decoder and the trained encoder.
Essentially, when (i) the reconstruction error is low enough, (ii) the sparsity regime is aligned with the ambient dimension of $\M$, and (iii) the dictionary is incoherent enough, we can ensure proper manifold recovery in the subspace sense.
The coefficients $(z_i)_{i \in \substar}$ then vary continuously as $\bm{x}_m$ moves along $\M$, tracing out the manifold in the SAE's coordinate system.

\paragraph{From Capture to Tiling.} The result above highlights an ideal scenario, i.e., the features align with the ambient space of the manifold.
However, when the number of atoms allocated to a manifold exceeds its ambient dimension $k_i$, the SAE is no longer constrained to reuse a fixed group and may assign different atoms to different regions of $\mathcal{M}_i$.
Each atom then acts as a localized detector with a receptive field on the manifold: this mechanism is essentially the one popularly studied in neuroscience, wherein neurons are argued to be sensitive to different values of a concept, covering overall geometry via the population code~\citep{khona2022attractor, eichenbaum2018barlow}.
In line with this literature, we call this phenomenon \textbf{tiling}: localized features with overlapping support whose joint activity encodes position along the manifold.
As we show in Fig.~\ref{fig:ising_hypo}, tiling manifests in two qualitatively different forms: \textbf{shattering}, where active sets $\{\supp(\z(\x_m))\}_{\bm{x}_m \in \M}$ across $\M$ are nearly disjoint and atoms partition the manifold, and \textbf{dilution}, where active sets overlap substantially but no compact group of size $\leq k_i$ accounts for $\M$. We give operational definitions of both regimes in Appendix.~\ref{app:ising}.

\begin{figure}[t]
    \centering
    \includegraphics[width=\linewidth]{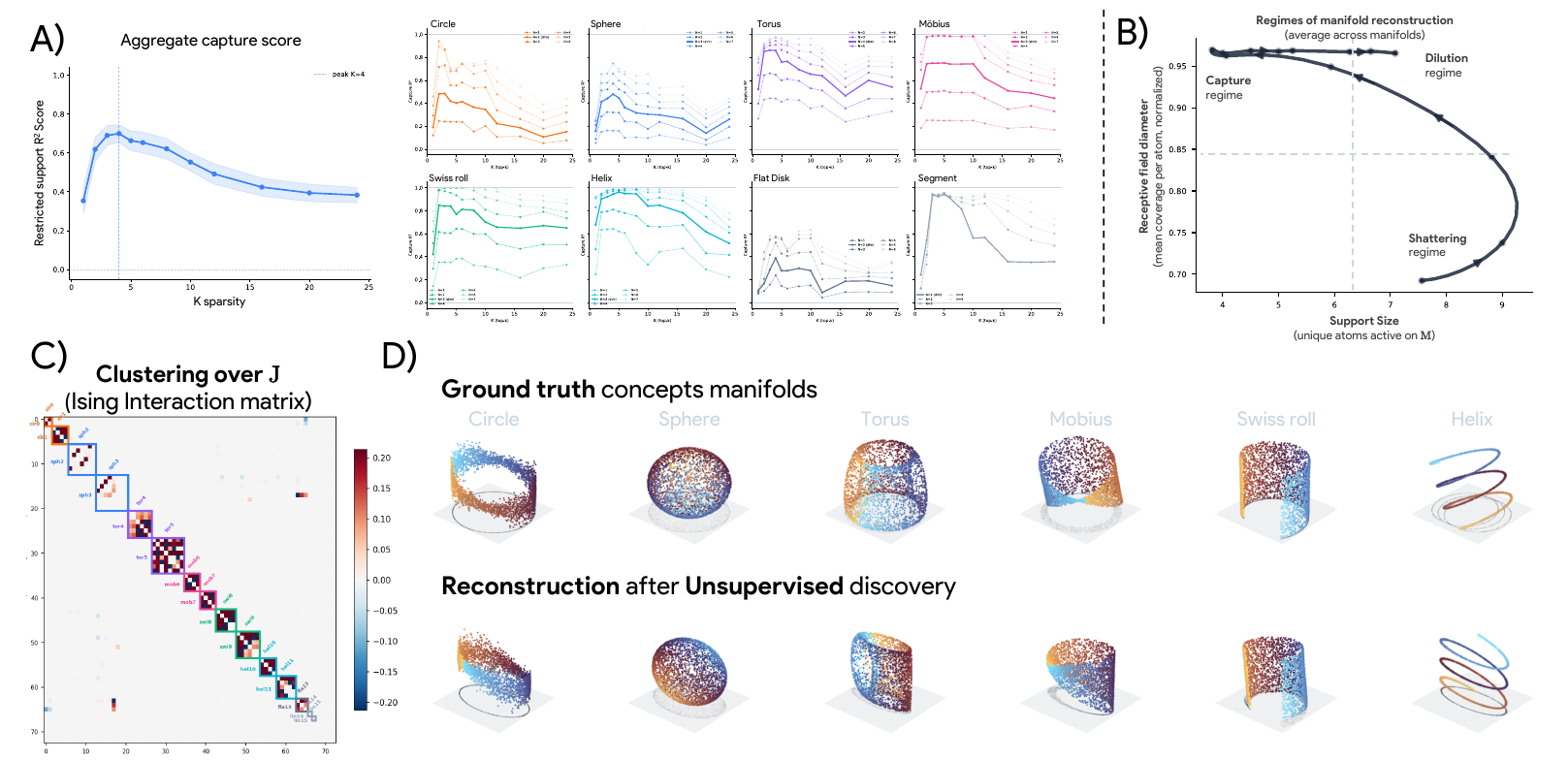}
    \caption{\small{\textbf{Synthetic validation of manifold capture.} We construct a controlled benchmark where observations are sparse mixtures of known manifolds embedded in $\mathbb{R}^{128}$ (dictionary width $c{=}512$, sparsity $k{=}4$) and make three observations.
    \textbf{A)} Subspace capture has a sparsity sweet spot. Restricted $R^2$ measures whether $k_i$ atoms suffice to reconstruct each manifold from the superposed codes. Capture peaks near $k = 4$ and degrades at both lower and higher sparsity. Per-manifold breakdowns (right) sweep the number of restricted atoms around each manifold's embedding dimension $k_i$.
    \textbf{B)} Increasing sparsity drives the SAE through three regimes. At low $k$, atoms are broadly shared and the manifold is shattered across unrelated groups. At intermediate $k$, a compact set of atoms spans each manifold's subspace (capture). At high $k$, many redundant atoms fire per point and individual atoms lose specificity (dilution). The phase diagram tracks this transition via support size and receptive field spread, averaged across all manifold types.
    \textbf{C--D)} Even outside the capture regime, manifold structure can be recovered post hoc. Fitting a pairwise Ising model on binarized codes yields a coupling matrix $\bm{J}$ whose block-diagonal structure aligns with the ground-truth manifold partition (\textbf{C}). Decoding through the recovered atom groups faithfully reconstructs the topology and geometry of all manifold types without supervision (\textbf{D}).}}
    \label{fig:synthetic}
\end{figure}

Regardless of whether the SAE is in the capture or tiling regime, the group of decoder atoms associated with a manifold is unknown and must be discovered from the codes alone.
To this end, one must use co-activation statistics: atoms that jointly represent a manifold fire together, or in smooth succession, across inputs on $\mathcal{M}_i$.
Raw co-activation of course confounds two distinct sources of statistical dependence: structural co-activation (atoms that span or tile the same manifold) and correlational co-occurrence (concepts that tend to appear together in the data).
It is also dominated by atoms that fire universally, which co-activate with everything without carrying manifold-specific information.

\subsection{Ising Pairings and Regimes of Manifold Representation}

To disentangle structural co-activation from spurious correlations, we model the joint activation statistics of SAE features using a pairwise Ising model over binarized codes~\citep{ising1925beitrag}. 
Let $s_i = 2*\mathbf{1}[z_i > 0] - 1$ denote whether atom $i$ is active. We define
\begin{equation}
p(s)\propto \exp\Big(\sum_{i<j} J_{ij}s_is_j + \sum_i h_i s_i\Big),
\end{equation}
where the fields $h_i$ absorb marginal firing rates and the couplings $J_{ij}$ capture \emph{direct} interactions between atoms after conditioning on the rest of the population.

This formulation isolates the dependencies that arise from atoms jointly representing a manifold. Atoms that fire frequently across all inputs are explained by large $h_i$ but exhibit weak couplings, while indirect correlations induced by superposition are factored out by construction. As a result, $J$ provides a more faithful representation of the functional relationships between features than raw co-activation or decoder similarity.

Importantly, the \emph{sign} and structure of the couplings reflect how a manifold is represented by the SAE (Figure \ref{fig:ising_hypo}. In the \textbf{capture} regime, a fixed set of atoms spans the manifold and co-activate consistently across inputs, yielding predominantly positive couplings within the group. In the \textbf{tiling} (shattering) regime, atoms specialize to distinct regions of the manifold and rarely activate together, leading to strong negative couplings that encode mutual exclusion. In the intermediate \textbf{dilution} regime, redundant and overlapping atoms produce a mixture of positive and negative interactions, resulting in a heterogeneous coupling structure.

These regimes therefore induce distinct signatures in the interaction matrix $J$. Rather than identifying manifolds through geometric similarity of decoder directions, we can instead recover them as \emph{communities of atoms with strong pairwise interactions}, irrespective of whether those interactions are cooperative or inhibitory. This perspective reframes manifold discovery as a problem of uncovering structured dependencies in feature activations, which we operationalize in Sec.~5.

\subsection{A Toy Model of Manifold Superposition}
\label{sec:toy_model}

To validate the framework of Sec.~\ref{sec:formalization} in a controlled setting, we construct a synthetic benchmark where the ground-truth manifolds, their ambient subspaces, and the sparse mixing process are all known by construction (Fig.~\ref{fig:synthetic}).
%
Specifically, we define eight manifold types spanning a range of topologies and intrinsic dimensions: circles, spheres, tori, M\"obius strips, Swiss rolls, helices, flat disks, and line segments. Each instance is embedded into $\mathbb{R}^d$ via a random orthonormal matrix $\bm{V}_i \in \mathbb{R}^{k_i \times d}$ and isotropically rescaled to unit RMS norm, preserving all geometric relationships. We instantiate six parameter variants per type (48 instances total), generate observations $\bm{x} = \sum_{i \in S} \bm{z}_i \bm{V}_i + \bm{\epsilon}$ following Defn.~\ref{def:mrh}, and train TopK SAEs across a range of sparsity budgets; see App.~\ref{app:synthetic} for details.

\paragraph{Results.}
Three findings emerge from this controlled setting (Fig.~\ref{fig:synthetic}). 
(\textbf{\textit{i}})~Subspace capture has a sparsity sweet-spot.
(\textbf{\textit{ii}})~Increasing sparsity drives the SAE through the three reconstruction regimes hypothesized in Section~\ref{sec:formalization}. 
(\textbf{\textit{iii}})~Even outside the capture regime, manifold structure can be recovered post hoc.

\section{Characterizing Manifold Capture in LLMs}
\label{sec:characterizing}

\begin{figure}[t]
    \centering
    \includegraphics[width=1.0\linewidth]{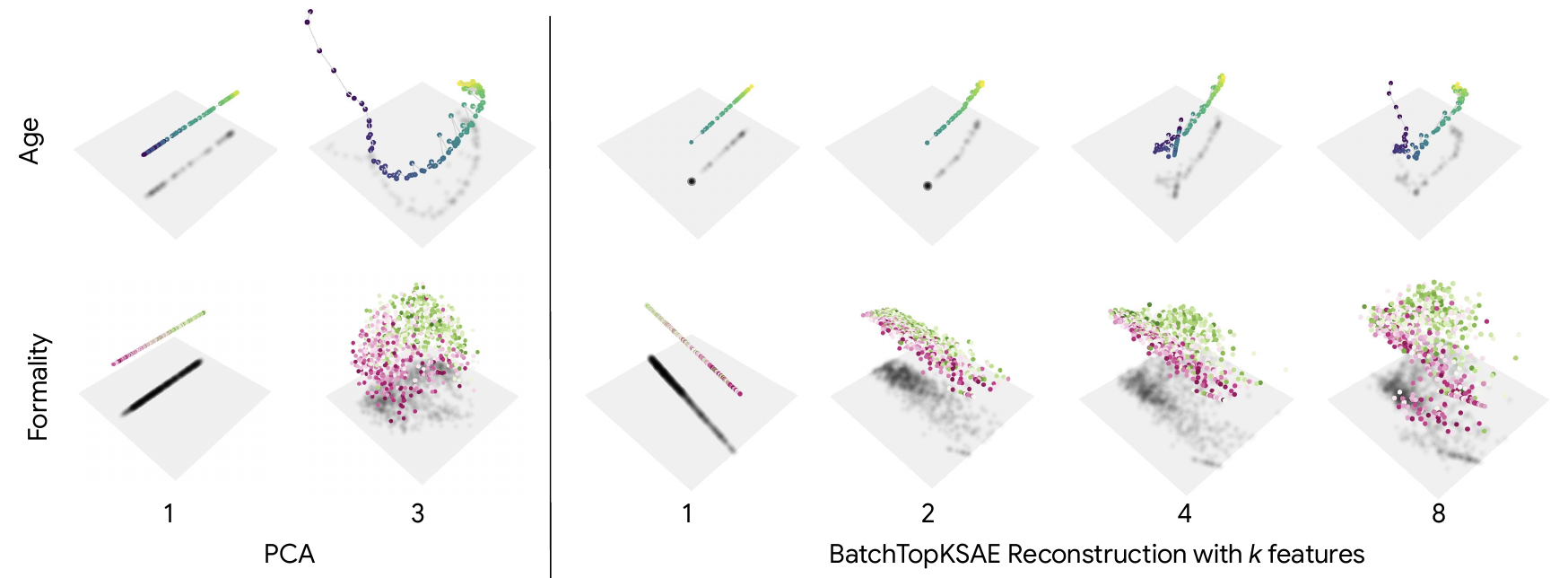}
    \caption{\small{\textbf{Piecewise-linear approximation of manifold geometry.} (Left) PCA projections of Llama3.1-8B activations show that manifolds are well described by a small number of global components. (Right) Reconstructing from increasing numbers of SAE features approximates the manifold in a piecewise-linear fashion: each feature captures a local region, and their union progressively covers the full geometry.
    }}
    \label{fig:feature_recon}
\end{figure}

We now aim to confirm our framework and findings from the synthetic setup in a more realistic setup.
To this end, we use representations from Llama3.1-8B at the residual stream of layer 19. 
We train five SAE architectures: Standard ($\ell_1$), JumpReLU~\citep{rajamanoharan2024jumping}, TopK~\citep{gao2024scaling}, BatchTopK~\citep{bussmann2024batchtopk}, and Matryoshka~\citep{bussmann2025learning}, with expansion factors of 8 and 16 and sparsities of 64, 128, and 256, on 500M tokens of The Pile~\citep{monology2021pile-uncopyrighted}; we only analyze SAEs achieving variance explained above 0.85. 
See App.~\ref{app:manifold_details} for further details.

\textbf{SAEs do not achieve compact capture.}
We apply the same restricted $R^2$ protocol as in Sec.~\ref{sec:toy_model}: for each manifold, we greedily select atoms by residual variance explained and measure reconstruction quality as a function of support size. Fig.~\ref{fig:subspace} shows the result averaged across manifolds and architectures. Variance explained grows with the number of restricted features but plateaus at a support size well beyond each manifold's ambient dimension. This indicates that current SAEs do not allocate a compact atom group whose span contains the manifold. Instead, the geometry is diluted across a larger, partially redundant set of features, placing SAEs in the dilution regime identified in Sec.~\ref{sec:toy_model}.

\begin{wrapfigure}{r}{0.5\textwidth}
    \centering
    \includegraphics[width=0.95\linewidth]{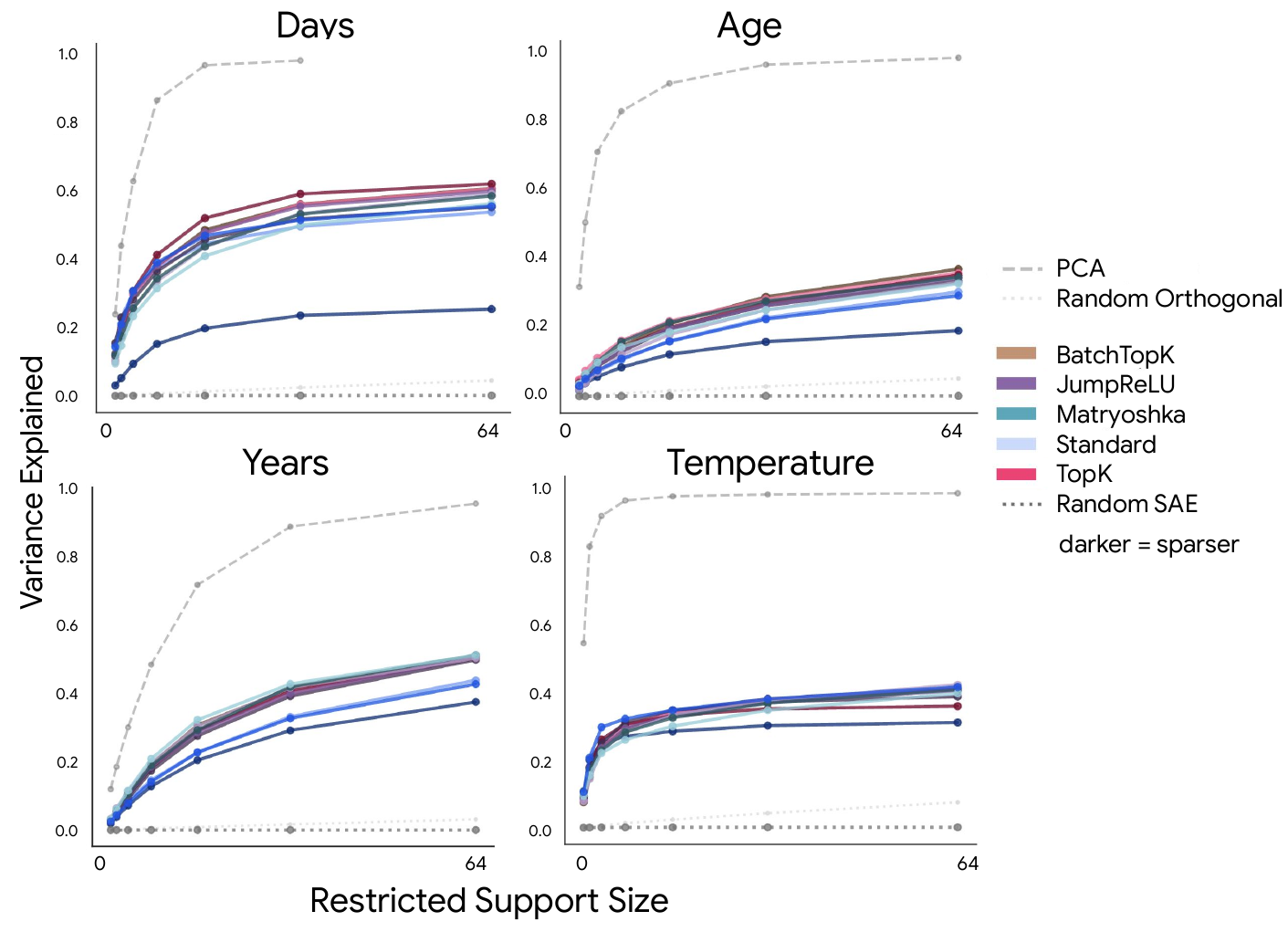}
    \caption{\small{\textbf{Subspace capture on Llama3.1-8B.} Variance explained as a function of the number of restricted features, averaged across manifolds and SAE architectures. Performance increases with support size but plateaus well beyond the manifold's ambient dimension, indicating that current SAEs are in the Dilution regime identified in Sec.~\ref{sec:toy_model}.}}
    \label{fig:subspace}
\end{wrapfigure}

\textbf{Features tile manifolds as localized detectors.}
If SAEs do not capture manifolds compactly, how do they represent them? Fig.~\ref{fig:feature_recon} contrasts PCA projections of the raw activations (which reveal smooth, low-dimensional geometry) with SAE reconstructions using increasing numbers of features. Individual features reconstruct local patches of the manifold in a piecewise-linear fashion, and their union progressively covers the full geometry. 
These results place current SAEs in an intermediate regime between ideal subspace capture---where a small, fixed set of atoms spans the manifold---and shattering---where localized features cover different regions of the geometry. 
While the manifold structure is preserved, it is fragmented across many features, consistent with the dilution behavior from Sec.~\ref{sec:toy_model}.

\begin{figure}
    \centering
    \includegraphics[width=1.0\linewidth]{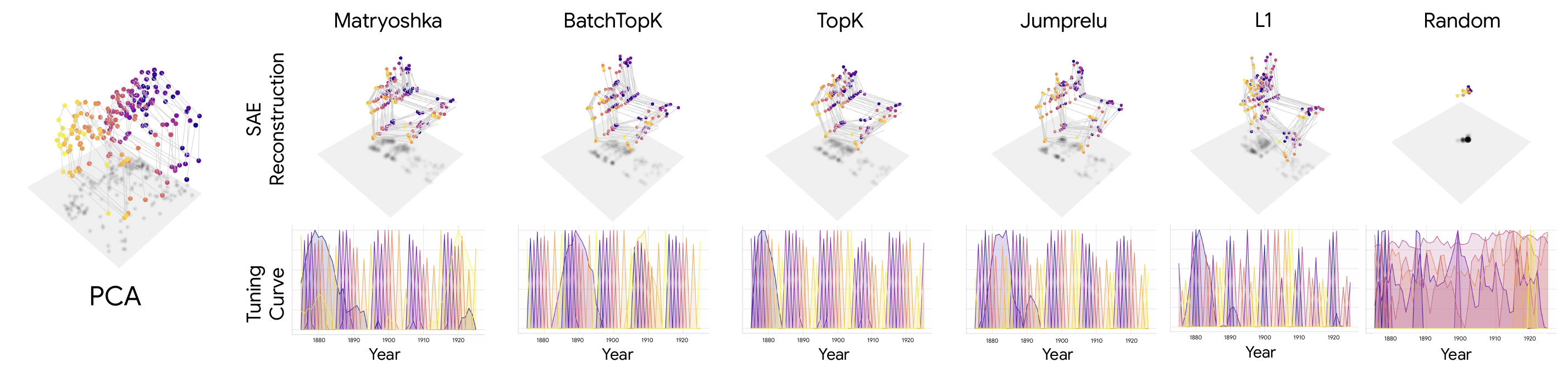}
    \caption{\small{\textbf{SAE features tile manifolds with tuning curves reminiscent of population coding.} Activations of the top features as a function of position along the ``years'' manifold. Each feature exhibits a localized, smooth activation profile covering a restricted region of the manifold, with overlapping support across features. For the years manifold, most SAEs learn features selective to the `ones' digit (activating periodically every 10 years) alongside features encoding the decade. These patterns are reminiscent of neural tuning curves in biological population codes, where no single neuron encodes the full variable but the population's joint activity traces out the underlying geometry.}}
    \label{fig:years_recon}
\end{figure}

\textbf{Tiling selectivity.}
This shattering effect is further made clear when analyzing feature activations as a function of the manifold concept.  In Fig.~\ref{fig:years_recon}, we plot the activations of the top-10 features for each SAE on the \texttt{years} manifold.
We observe that features exhibit localized activation patterns, responses vary smoothly across the manifold, and multiple features cover overlapping regions of the variable. 
In particular, for years, we can see that most SAEs learn individual features that represent the ones digit of the year, activating periodically every 10 years, as well as other features that carry information about the decade. 
These patterns are reminiscent of neural tuning curves~\citep{pouget1999narrow,hubel1968receptive,georgopoulos1986neuronal}, where each feature responds to a restricted region of the manifold.
Figure \ref{fig:days_receptive} visualizes the receptive fields of the top features for each SAE on the days of the week manifold, highlighting the selectivity of features in the ambient space within which the manifold lives.

We further explore the selectivity of features by plotting SAE feature activations in the ambient space (defined by the top 3 principal components) a manifold lives in. We sample points in the ambient space, decompose them with the SAE, and color those points by their max activating features. We weight the size of these points by using the reconstruction error of the SAE to model the ambient space as a probability distribution of possible manifold points. In Figure \ref{fig:days_receptive}, we see that features exhibit selectivity for each day of the week, with the different assumptions made by each SAE resulting (e.g. angular separability vs linear separability) visible in how the ambient space is shattered by the features. 

\begin{figure}
    \centering
    \includegraphics[width=1.0\linewidth]{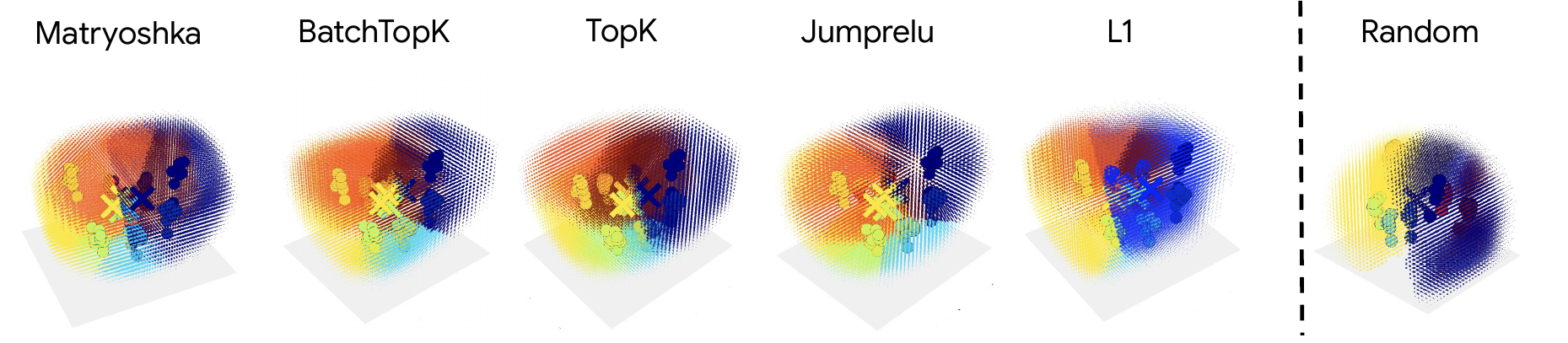}
    \caption{\small{Receptive field plots for different SAE architectures on the days of week manifold. Sampled points in the ambient space of the manifold are colored by their highest activating SAE feature, highlighting feature selectivity in the ambient space as well as the architectural biases of different SAEs (e.g., angular separability in Top-K SAEs and linear separability in $L_1$).}}
    \label{fig:days_receptive}
\end{figure}

Overall, the observations in this section strongly support a tiling model of representation: manifolds are encoded by collections of localized features whose joint activity captures the underlying geometry.
A second critical implication is that individual features will only offer a narrow view of what concept an SAE is trying to capture: only the group of features that tile a manifold as a whole carries geometric meaning. 
Interpretability in this regime thus requires reasoning about subspaces, not about individual dictionary elements.

\vspace{-10pt}
\begin{goodfirebox}{A Mismatch Between Featurizers and Geometric Structure.}
SAEs were designed around sparsity in an almost-orthogonal dictionary, with the linear representation hypothesis as their guiding intuition: features compose by addition and are recovered as directions. 
Recent refinements admit multidimensional features, but the featurizer has not followed: atoms remain one-dimensional, and the loss rewards reconstruction without rewarding any coherent treatment of the geometric objects that underlie it. The structures we identify sit outside this design point---curved rather than flat, equipped with non-trivial topology, and organized into objects no single atom captures.

~~

Post-hoc recovery, as we will show now, is possible but unreliable: mixed-selectivity features muddle the co-activation signal, and groupings extracted from Ising couplings are only as good as the underlying tiling permits. Until featurizers are developed that treat geometric objects as the primitive (rather than directions to be assembled into geometry after the fact) post-hoc analysis remains the best available tool, but should be read as a workaround, not a solution.
\end{goodfirebox}

\begin{figure}[t]
    \centering
    \includegraphics[width=1.0\linewidth]{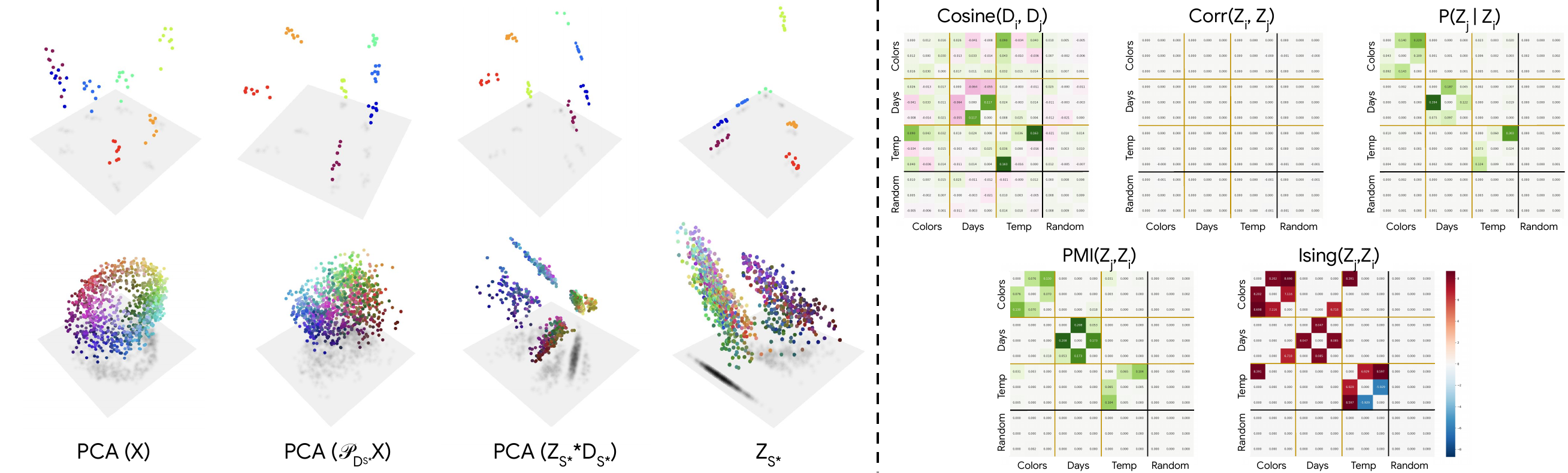}
    \caption{\small{\textbf{Reading manifold geometry from feature groups.} (Left) Four views of the days and colors manifolds using the top 3 supervised features per manifold: PCA of activations (ground truth), PCA of the projection onto the decoder subspace spanned by the group, PCA of partial code reconstructions, and raw feature activations as coordinates. Projecting onto the decoder subspace most faithfully recovers the continuous geometry. (Right) Pairwise feature similarity under five metrics. Ising couplings and conditional co-activation produce the clearest block-diagonal structure aligned with ground-truth manifold assignments.}}
    \label{fig:pca_clustering}
\end{figure}

\begin{figure}[h]
    \centering
    \includegraphics[width=\linewidth]{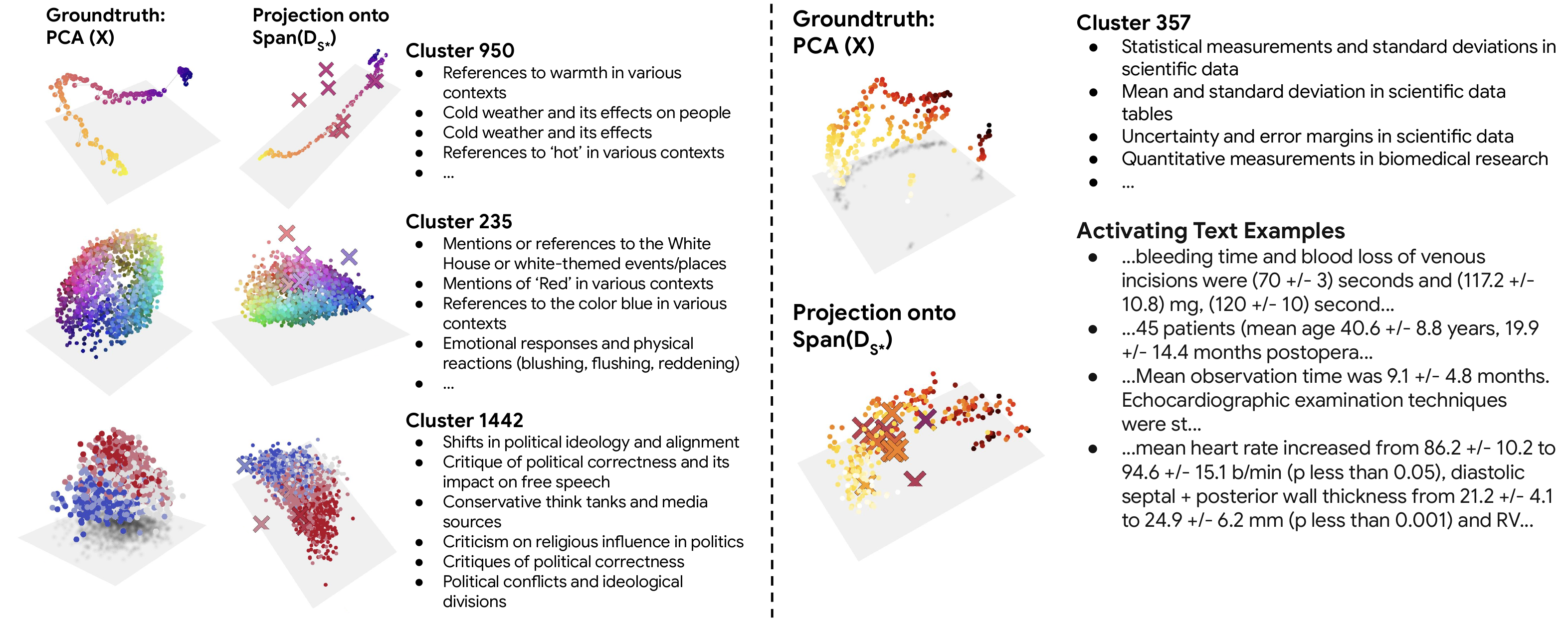}
    \caption{\small{\textbf{Unsupervised Discovery from SAE Codes.} (Left) The Ising-pipeline recovers known manifolds (temperature, colors, political bias) as distinct feature communities. (Right) The same pipeline surfaces a novel manifold encoding epistemic uncertainty in scientific contexts, demonstrating its utility for generating hypotheses beyond known structures.
    }}
    \label{fig:unsup_known}
\end{figure}

\section{Unsupervised Manifold Discovery}
\label{sec:discovery}

The results of Sec.~\ref{sec:characterizing} confirm that SAEs distribute manifold geometry across many localized features. Recovering coherent geometric objects therefore requires post-hoc analysis that groups related atoms without prior knowledge of the underlying manifolds. 
We thus now evaluate candidate grouping strategies and demonstrate that the Ising-model introduced in Sec.~\ref{sec:formalization} transfers from the synthetic setting to real language model representations.

\textbf{Which similarity metric can recover manifold groups.}
A natural starting point is to cluster features by decoder cosine similarity, as explored in prior work~\citep{engels2025}. However, under subspace capture, the atoms spanning a manifold's ambient subspace may be nearly orthogonal, and under tiling, atoms covering adjacent but non-overlapping regions of the manifold need not have similar decoder directions. Decoder geometry thus carries no privileged information about the manifold topology that features collectively tile.
Co-activation statistics offer a more principled alternative: features that jointly represent a manifold fire together, and in the shattering case the features that jointly represent a manifold have extreme inhibition against each other. 
We compare five similarity measures for constructing feature affinity graphs: (\textbf{\textit{i}})~decoder cosine similarity, (\textbf{\textit{ii}})~conditional co-activation probability, (\textbf{\textit{iii}})~Pearson correlation of activation magnitudes, (\textbf{\textit{iv}})~pointwise mutual information, and (\textbf{\textit{v}})~Ising pairwise couplings. To evaluate each metric, we use the supervised feature selection pipeline as ground truth: for three manifolds (colors, days, and temperature), we take the top three features per manifold and compute pairwise similarity under each metric. Fig.~\ref{fig:pca_clustering} (right) visualizes the resulting affinity matrices. A metric succeeds if it produces clear block-diagonal structure with high within-manifold and low cross-manifold similarity. \textbf{Ising couplings and conditional co-activation yield the cleanest separation}, while decoder cosine similarity and Pearson correlation fail to recover the block structure, consistent with the observation that manifold membership is a functional property (which atoms co-activate) rather than a geometric one (where atoms point).

\textbf{Unsupervised discovery pipeline.}
We apply the Ising-pipeline to a BatchTopK SAE trained on Llama3.1-8B (layer 19, expansion $\times 8$, $k = 64$). Fig.~\ref{fig:unsup_known} (left) confirms that the procedure recovers the supervised manifolds identified in Sec.~\ref{sec:qualitative_proof}: temperature, colors, and political bias emerge as distinct communities with coherent geometric structure. 
Beyond recovering known manifolds, the pipeline also surfaces \textit{novel} geometric structures. Fig.~\ref{fig:unsup_known} (right) shows a previously unidentified manifold related to epistemic uncertainty, encoding the degree of measurement error and imprecision in scientific contexts. This demonstrates that the Ising-based discovery pipeline can serve as a tool for unsupervised manifold discovery.

\section{Conclusion}
The presence of structured, nonlinear geometries in model representations suggests that the fundamental unit of interpretation need not be isolated directions.
While sparse autoencoders can, in principle, represent such structures, we show that in practice they do so in a fragmented manner: manifolds are not captured as coherent subspaces, but are instead tiled across many localized, partially redundant features. 
This preserves geometry only implicitly, obscuring it at the level of individual features and limiting the reliability of direction-based interpretability.
Moving forward, we argue that interpretability should be reframed around the recovery and manipulation of geometric structures rather than individual directions. 
This includes both developing featurization methods that explicitly target manifolds, and designing analysis tools that operate on groups of features as coherent units. 

More broadly, our results suggest that understanding neural networks requires shifting from a dictionary of concepts to a geometry of representations---where meaning is encoded not in single atoms, but in the structure they collectively induce.

\section*{Acknowledgments} 

The authors thank Thomas Icard and the Mechanisms team at Goodfire, David Klindt, Aaron Mueller, Demba Ba, Sumedh Hindupur, Valerie Costa, and Ren Makino for helpful discussions during the course of this project.

\bibliography{main}
\bibliographystyle{colm2026_conference}

\clearpage
\appendix

\section{Extended Related Work}
\label{app:related}
\paragraph{Neuroscience.} The existence of such geometry implies higher-order structure on top of individual concepts: groups of features that co-vary continuously, forming coherent geometric objects. This perspective connects to a foundational principle in neuroscience: continuous variables are typically encoded not by single neurons, but by populations of neurons with localized, overlapping receptive fields that collectively tile the underlying space \citep{pouget2000information}. Place cells in the hippocampus tile physical space, each firing in a circumscribed spatial region, so that the animal's location is encoded by the pattern of co-active cells \citep{o1971hippocampus}. Orientation-selective neurons in primary visual cortex tile the space of edge angles \citep{hubel1962receptive}. In many such cases, no single neuron encodes the full concept; rather, the population's joint activity maps out the underlying geometry, and the concept's value can be decoded from the population response \citep{georgopoulos1986neuronal}. If neural network representations recapitulate this coding strategy, one would expect manifolds to be represented by many localized features whose activations tile the geometry, rather than by single features aligned with global directions.

\paragraph{Geometry of Neural population.}
Prior work provides evidence that situates a mixture of Manifold within a tradition of geometric study of neural network. 
First (i), early work found that ReLU-based architectures partition input space into convex polyhedral linear regions (often unbounded). Theoretical analyses of this partition structure have established theoretical results on the number of linear regions~\citep{montufar2014number,telgarsky2015representation,serra2018bounding,raghu2017expressive,balestriero2018spline,balestriero2020mad}. Empirically, prior studies have found that trained networks realize far fewer regions than the maximal theoretical counts~\citep{hanin2019complexity,zhang2020empirical}; related interpretability work have  exploited this polyhedral structure by enumerating regions to extract exact piecewise-linear rules~\citep{black2022interpreting,chu2018exact}. Second (ii), in representation space, recent analyses demonstrate a convex organization of activations and architecture-specific convex projections~\citep{tvetkova2025convex,fel2025into}. This observation dovetails with results in population geometry indicating that network activity concentrates on low-dimensional manifolds with structured variability~\citep{chung2021neural,cohen2020separability,engels2025,sarfati2026shape}.%
Third (iii), in language models, recent work has shown that categorical and hierarchical concepts admit polytopal encodings whose geometric relations mirror semantic relations~\citep{park2024geometry,park2026information}. 
%

\paragraph{Sparse Autoencoders.}
In recent years, SAEs have resurfaced as a popular implementation of sparse coding and dictionary learning to provide concept-level explanations for neural networks \citep{olshausen1997sparse, bricken2023monosemanticity, cunningham2023sparse, gao2024scaling, rajamanoharan2024jumping, bussmann2024batchtopk, fel2025archetypal}. Advancements beyond ReLU SAEs have included TopK \citep{gao2024scaling}, BatchTopK \citep{bussmann2024batchtopk}, and JumpReLU \citep{rajamanoharan2024jumping} nonlinearities. Archetypal SAEs \citep{fel2025archetypal} address the algorithmic instability of SAEs, and Matryoshka SAEs \citep{bussmann2025learning} and MP-SAEs \citep{costa2025flat} learn hierarchical concept dictionaries. TFA \citep{lubana2025priors} and T-SAEs \citep{bhalla2026temporal} incorporate temporal information into dictionary learning methods, allowing for recovery of temporally abstract features.

While the use of SAEs is motivated by the LRH, a growing body of recent work has challenged this assumption by revealing hierarchical concepts, dense, "onion-like" representations, and multi-dimensional contextual representations \citep{wattenberg2024relational, park2024geometry, csordas2024recurrent, engels2025,michaud2025understanding,wollschlager2025geometry}. Along another vein, studies have highlighted the lack of intended causal efficacy of SAEs \citep{wu2025axbench, bhalla2024towards}, their limited utility for probing \citep{kantamneni2025sparseautoencoder,karvonen2024measuring}, and problems with feature splitting and absorption \citep{chanin2024absorption}. Many of these issues can be partially explained by representations being additive mixtures of manifolds and SAEs tiling this representation space, understanding this tiling behavior requires tools from a richer geometric tradition.

\paragraph{Manifold learning \& Subspace clustering.}
\label{app:scc}
A natural framework for understanding how SAEs tile representation space comes from the literature on Manifold Learning and Sparse Subspace Clustering (SSC). Those long-standing line of work has studied how to recover low-dimensional structure from high-dimensional data through sparse reconstruction primitives, and contextualizes our framework. 
Nonlinear manifold learning reconstructs each datapoint from a small set of neighbors on the manifold itself, using either global geometry~\citep{tenenbaum2000global,silva2002global}, local linear patches~\citep{roweis2000nonlinear,vladymyrov2013locally}, spectral embeddings of a neighborhood graph~\citep{belkin2001laplacian,coifman2006diffusion}, curvature-aware local reconstructions~\citep{donoho2003hessian,zhang2004principal}, or topological aggregations~\citep{mueller2022geometric}. 
On the other hand, Subspace clustering and its nonlinear extensions (see this survey by~\citealt{abdolali2021beyond}) instead represent each datapoint as a sparse combination of other datapoints in the same subspace and partition the resulting affinity graph by spectral clustering~\citep{elhamifar2013sparse,liu2010robust,liu2012robust,soltanolkotabi2014robust,you2016scalable,li2017structured}. 
Nonlinear extensions adapt this primitive to data on a union of manifolds via locality preservation and tangent estimation~\citep{elhamifar2011sparse}, kernels~\citep{patel2014kernel}, neural networks~\citep{ji2017deep,li2022neural}, or matching pursuits~\citep{tschannen2018noisy}, though even the deep variants have been shown to be ill-posed under their own assumptions~\citep{haeffele2020critique}. 
A third strand explicitly bridges sparse coding and manifold learning by penalizing locality so that sparse codes select only nearby atoms~\citep{yu2009nonlinear,wang2010locality} or by tying the dictionary to manifold geometry directly~\citep{chen2018sparse}. All of these methods share a common generative assumption: each observation is associated with a single subspace or manifold via a latent label,
\begin{equation}
\label{eq:single-label-dgp}
    \bm{x}_i \;=\; \bm{U}_{\ell(i)} \, \bm{z}_i \;+\; \bm{\epsilon}_i,
    \qquad \ell(i) \in \{1, \dots, c\},
\end{equation}
and the goal is to recover the partition $\ell$. Our additive mixture of manifolds (Defn.~\ref{def:mrh}) departs from~\eqref{eq:single-label-dgp} by allowing each observation to participate in \emph{multiple} manifolds simultaneously, $\bm{x} = \sum_{i \in S} \bm{m}_i$ with $|S| \ll m$. This single change has two consequences that make the prior toolkit inapplicable \emph{as is}: (i) classical self-expression $\bm{x}_i = \bm{X} \bm{c}_i$ cannot be subspace-preserving, because the coefficients must reach across every manifold active in $\bm{x}_i$; and (ii) the unit of clustering shifts from datapoints to dictionary atoms, since no point-level label $\ell(\bm{x})$ exists. Our Ising affinity is therefore an $c \times c$ object over learned features rather than an $n \times n$ object over points, and the spectral-clustering pipeline of nonlinear SC does not transfer to our regime. We see this not as a rejection of the prior literature but as identifying the missing additive ingredient required in our assumptions.

\section{The Ubiquity of Manifolds}
\label{app:manifold_details}                            All manifolds are evaluated using last-token activations from Llama-3.1-8B at layer 19 ($d = 4096$). Table~\ref{tab:manifolds} summarizes each manifold's prompt template, sample count, ground-truth labels, and source dataset.

\begin{table}[h]                                        \centering
\small                                                  \caption{Manifold datasets used for evaluation. All activations are extracted at the last token position.}
\label{tab:manifolds}                                   \begin{tabular}{lllll}
\toprule                                                \textbf{Manifold} & \textbf{Geometry} & \textbf{$n$} & \textbf{Prompt template / source} \\
\midrule 
colors & paraboloid& ${\sim}900$ & ``The hex code \{h\_code\} is for the color'' \\ 
temperature & line & 150 & ``Today it's \{f\}   
degrees Fahrenheit outside'' \\
age & line & 99 & ``They are \{age\} years old.''  \\
geography & hierarchical tree & ${\sim}4{,}000$ & ``The geographical coordinates \{lat, lon\} are in the country of'' \\
days & circle & 420 & ``It's \{time\} on {day}" \\ 
years & helix & 199 & ``The date is {year}'' \\
formality & line & 1{,}000 & \cite{pavlick2016empirical} \\                     
sent\_length & line & 5{,}000 & WikiText \\ 
politic bias & 1D continuous & varies & GPT-5 augmentations of \cite{cajcodes_political_bias} \\ 
\bottomrule                                            \end{tabular}                                         \end{table}        
  
\subsection{Steering Details}
\label{app:steering_details}                            To verify that manifold structure is causally relevant to model behavior, we perform activation patching along the principal components of each manifold's activations (Fig.~\ref{fig:manis}, right).
\paragraph{Setup.} For each continuous manifold, we fit PCA on the cached layer-19 activations, retaining enough components to explain 90\% of      
variance. We then select a base prompt near the manifold's midpoint and construct a sweep by binning the manifold's primary continuous label (e.g., 
fahrenheit for temperature, hue for colors) into 5-10 equal-width bins. For each bin, we compute the PCA centroid and linearly interpolate 5-10 points    
between consecutive centroids, yielding evenly spaced intervention points per manifold.
\paragraph{Intervention.} For each intervention point, we compute the PCA-space delta from the manifold mean, project it back to activation space, 
and add it to the base prompt's layer-19 activation during a forward pass. We then collect next-token logits and track the probabilities of a set of 
target tokens chosen to be semantically diagnostic of the underlying variable. Table~\ref{tab:steering_tasks} lists the task suffix and target tokens
for each manifold.

\begin{table}[h]
\centering
\small
\caption{Steering task configuration per manifold.}
\label{tab:steering_tasks}
\begin{tabular}{lp{3.5cm}p{6cm}}
\toprule
\textbf{Manifold} & \textbf{Task suffix} & \textbf{Target tokens}\\
\midrule
temperature & ``, which is'' & freezing, cold, chilly, nice, warm, hot, burning \\ 
age & `` They are'' & young, middle, old, elderly \\
 years & ``-01-01, which year is it?'' & 17, 18, 19, 20, 21 \\ 
 colors & (none) & red, orange, yellow, green, blue, indigo, purple \\ 
 day & ``. What day is it?'' & Monday, Tuesday, Wednesday, Thursday, Friday, Saturday, Sunday \\
 \bottomrule
\end{tabular}
\end{table}   

\subsection{SAE Training Details}
\label{app:sae_training}
All SAEs are trained on activations from Llama-3.1-8B at layer 19 (residual stream, $d = 4096$). Activations are harvested from 500M tokens of The 
  Pile (uncopyrighted) \citep{gao2020pile} using sequence length 4096.
\paragraph{Optimization.} All architectures use Adam with learning rate $10^{-4}$, no weight decay, and gradient clipping at max norm 1.0. Batch   
  size is 16{,}384 tokens. We use a linear warmup over the first 1
  epochs (though most checkpoints are taken at epoch 2 based on validation VE). Activations are auto-normalized by their mean $\ell_2$ norm prior to   
  training.       
  
\paragraph{Architecture-specific details.}

\begin{itemize}[leftmargin=*]
\item \textbf{TopK / BatchTopK}: Auxiliary loss weight 0.05 with 1
    \item \textbf{JumpReLU}: STE bandwidth $\varepsilon = 0.001$. Target L0 set to match $k$ of other architectures.                                      \item \textbf{Matryoshka}: Nested feature groups with geometrically spaced sizes ($d_\text{sae}/8$, $d_\text{sae}/8$, $d_\text{sae}/4$, remainder). Otherwise same as BatchTopK.
    \item \textbf{Standard ($\ell_1$)}: Sparsity weight $\lambda \in {0.03, 0.04, 0.1}$.
\end{itemize} 

\paragraph{Model selection.} Table~\ref{tab:sae_zoo} lists all trained SAEs. We retain SAEs achieving variance explained (VE) $> 0.85$ on held-out  activations for the main experiments. SpaDE ($\text{VE} \approx 0.62$) and MFA ($\text{VE} \approx 0.55$) are included for architectural comparison despite lower reconstruction quality.                                                                                         \begin{table}[h]
\centering
\small
\caption{SAE configurations. VE is variance explained at epoch 2. SAEs below the VE $> 0.85$ threshold (marked with $\dagger$) are included for    
  architectural comparison only.}                       \label{tab:sae_zoo}                                   \begin{tabular}{llccc}                                \toprule                                              \textbf{Architecture} & \textbf{Expansion} & \textbf{Sparsity} & \textbf{$d_\text{sae}$} & \textbf{VE} \\
  \midrule
BatchTopK & $\times 8$ & $k = 64$ & 32{,}768 & 0.858 \\
BatchTopK & $\times 8$ & $k = 128$ & 32{,}768 & 0.858\\                                                      
BatchTopK & $\times 8$ & $k = 256$ & 32{,}768 & 0.883 \\
BatchTopK & $\times 16$ & $k = 128$ & 65{,}536 & 0.849\\
TopK & $\times 8$ & $k = 64$ & 32{,}768 & 0.852 \\
TopK & $\times 8$ & $k = 128$ & 32{,}768 & 0.861 \\
TopK & $\times 8$ & $k = 256$ & 32{,}768 & 0.862 \\
TopK & $\times 16$ & $k = 64$ & 65{,}536 & 0.845 \\
TopK & $\times 16$ & $k = 128$ & 65{,}536 & 0.848 \\
JumpReLU & $\times 8$ & $k = 64$ & 32{,}768 & 0.841 \\
JumpReLU & $\times 8$ & $k = 128$ & 32{,}768 & 0.846 \\
JumpReLU & $\times 8$ & $k = 256$ & 32{,}768 & 0.867 \\
JumpReLU & $\times 16$ & $k = 64$ & 65{,}536 & 0.853 \\
JumpReLU & $\times 16$ & $k = 128$ & 65{,}536 & 0.852 \\
JumpReLU & $\times 16$ & $k = 256$ & 65{,}536 & 0.879 \\                                                      Matryoshka & $\times 8$ & $k = 64$ & 32{,}768 & 0.847 \\
Matryoshka & $\times 8$ & $k = 128$ & 32{,}768 & 0.851 \\
Matryoshka & $\times 8$ & $k = 256$ & 32{,}768 & 0.875 \\
Matryoshka & $\times 16$ & $k = 128$ & 65{,}536 & 0.854 \\
Matryoshka & $\times 16$ & $k = 256$ & 65{,}536 & 0.887\\
Standard & $\times 8$ & $\lambda = 0.03$ & 32{,}768 & 0.841\\
Standard & $\times 8$ & $\lambda = 0.1$ & 32{,}768 & 0.848\\
Standard & $\times 16$ & $\lambda = 0.03$ & 65{,}536 & 0.867\\
Standard & $\times 16$ & $\lambda = 0.04$ & 65{,}536 & 0.844\\
Standard & $\times 16$ & $\lambda = 0.1$ & 65{,}536 & 0.961\\
\bottomrule                                             \end{tabular}                                           \end{table}

\begin{figure}
    \centering
    \includegraphics[width=1.0\linewidth]{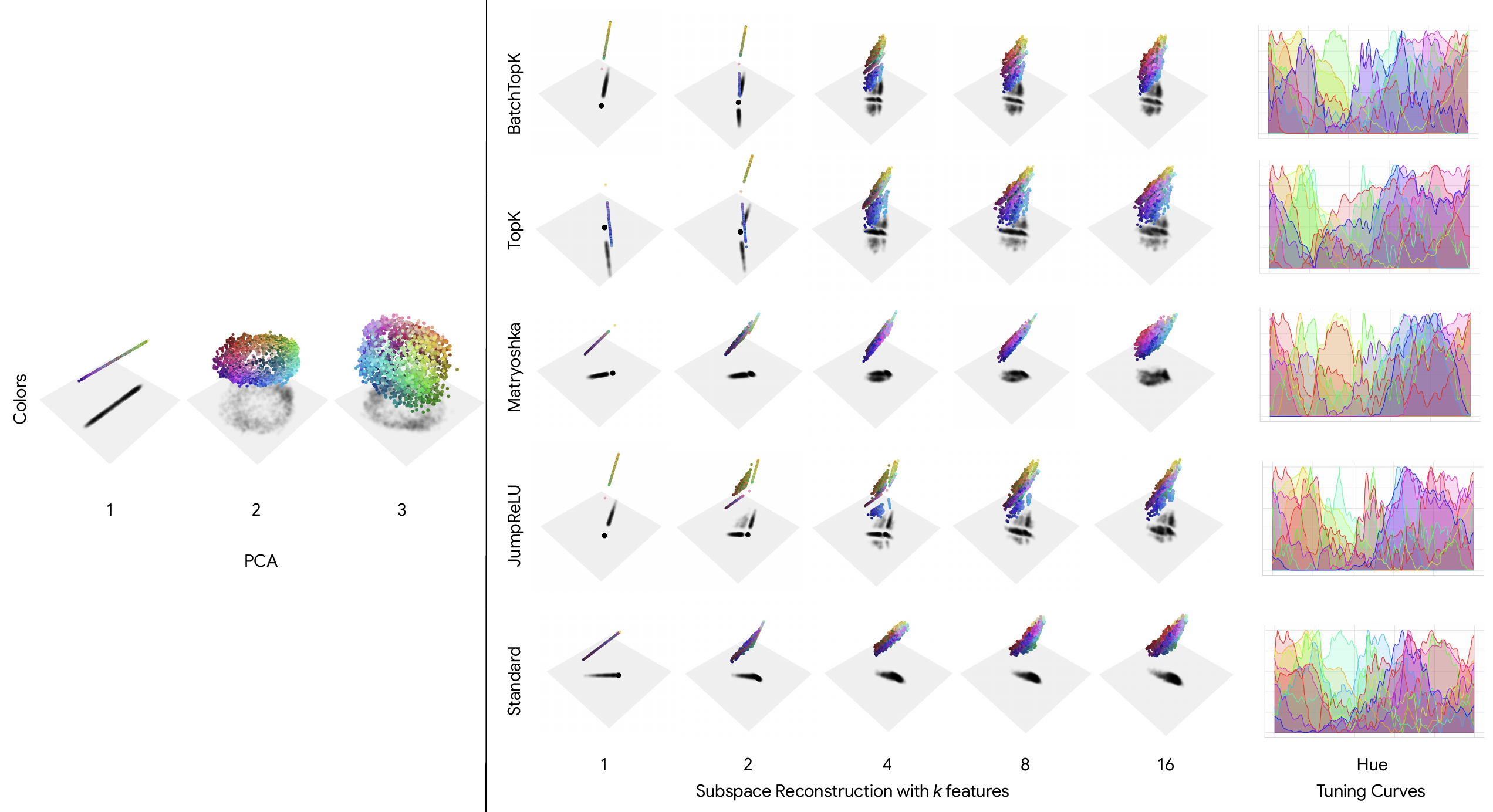}
    \caption{(Left) PCA projections show that manifolds are well-described by a small number of global components encoding semantic variation.
(Middle) SAE reconstructions using increasing numbers of features approximate the manifold in a piecewise-linear fashion, with individual features capturing local regions. (Right) Tuning curves highlight the mixed selectivity of features across the hue dimension of the color manifold.}
    \label{fig:colors}
\end{figure}

\begin{figure}
    \centering
    \includegraphics[width=1.0\linewidth]{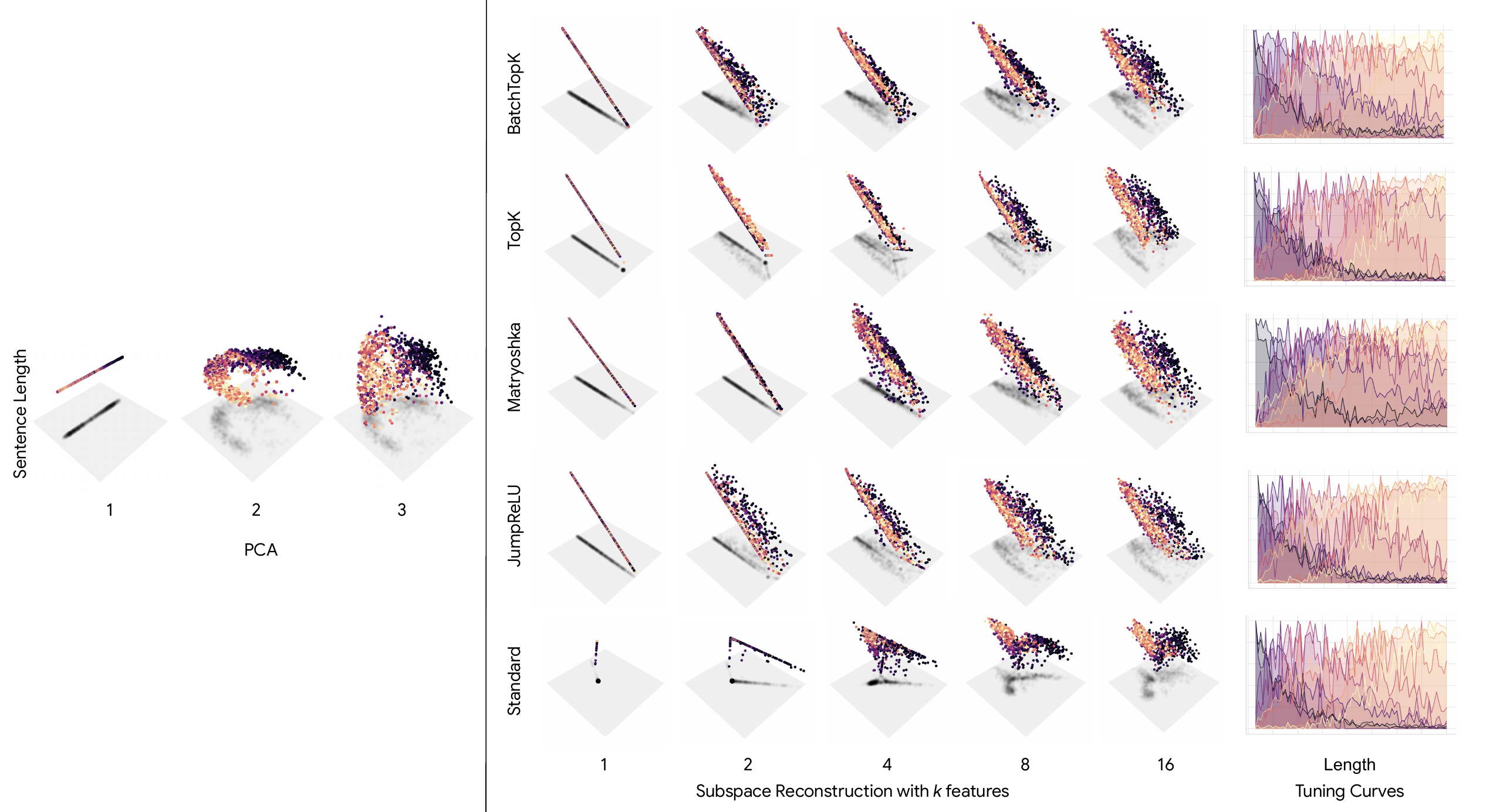}
    \caption{(Left) PCA projections show that manifolds are well-described by a small number of global components encoding semantic variation.
(Middle) SAE reconstructions using increasing numbers of features approximate the manifold in a piecewise-linear fashion, with individual features capturing local regions. (Right) Tuning curves highlight the mixed selectivity of features across sentence length.}
    \label{fig:sent_length}
\end{figure}

\subsection{Platonic Representations?}

We investigate whether different SAEs recover a shared underlying representation of the same manifold. To do so, we compare learned features across models using optimal transport (OT) in three spaces: decoder directions, code activations on random inputs, and code activations restricted to points lying on a given manifold.

Fig.~\ref{fig:feature_alignment} shows weak alignment when comparing decoder directions, particularly for point-based methods such as SpaDE, which differ substantially from direction-based SAEs. Similarly, OT applied to SAE activations on random training data reveals slightly more consistent structure across models, but no clear alignment. These results suggest that individual features are not stable objects: their representation depends strongly on architectural choices and training dynamics, in line with previous work \citep{fel2025archetypal, paulo2025sparse}.

In contrast, when restricting comparison to specific manifolds, we observe strong alignment across SAEs. Despite differences in individual features, the induced coordinate systems over the manifold are highly consistent, indicating that what is preserved across SAEs is not the features themselves, but the geometric structures they collectively encode.

\begin{figure}
    \centering
    \includegraphics[width=1.0\linewidth]{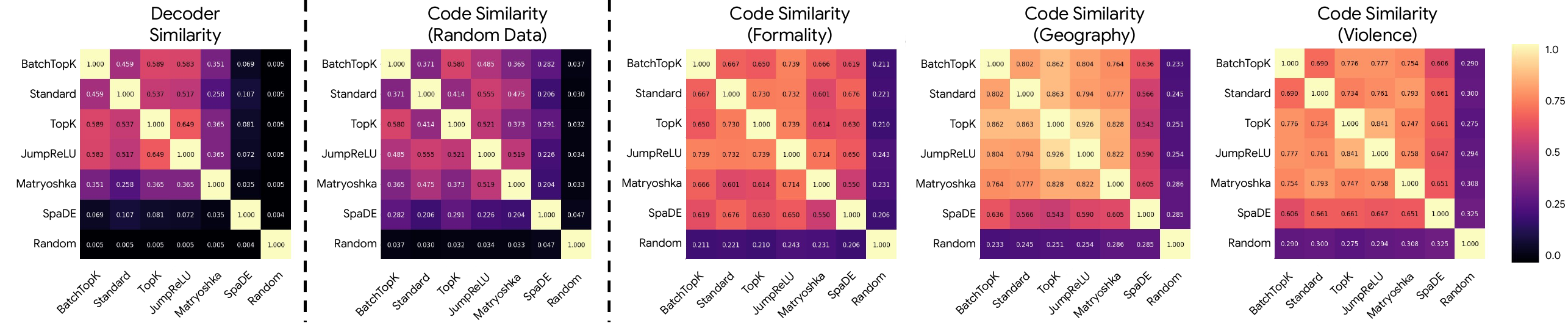}
    \caption{Similarity between features learned by different SAEs measured in decoder space (left), SAE code space (middle and right) for random data (middle) and specific manifolds (right).}
    \label{fig:feature_alignment}
\end{figure}

\clearpage

\section{Duality of Concept Geometry}
\label{app:duality}

\begin{figure}[h]
    \centering
    \includegraphics[width=0.8\linewidth]{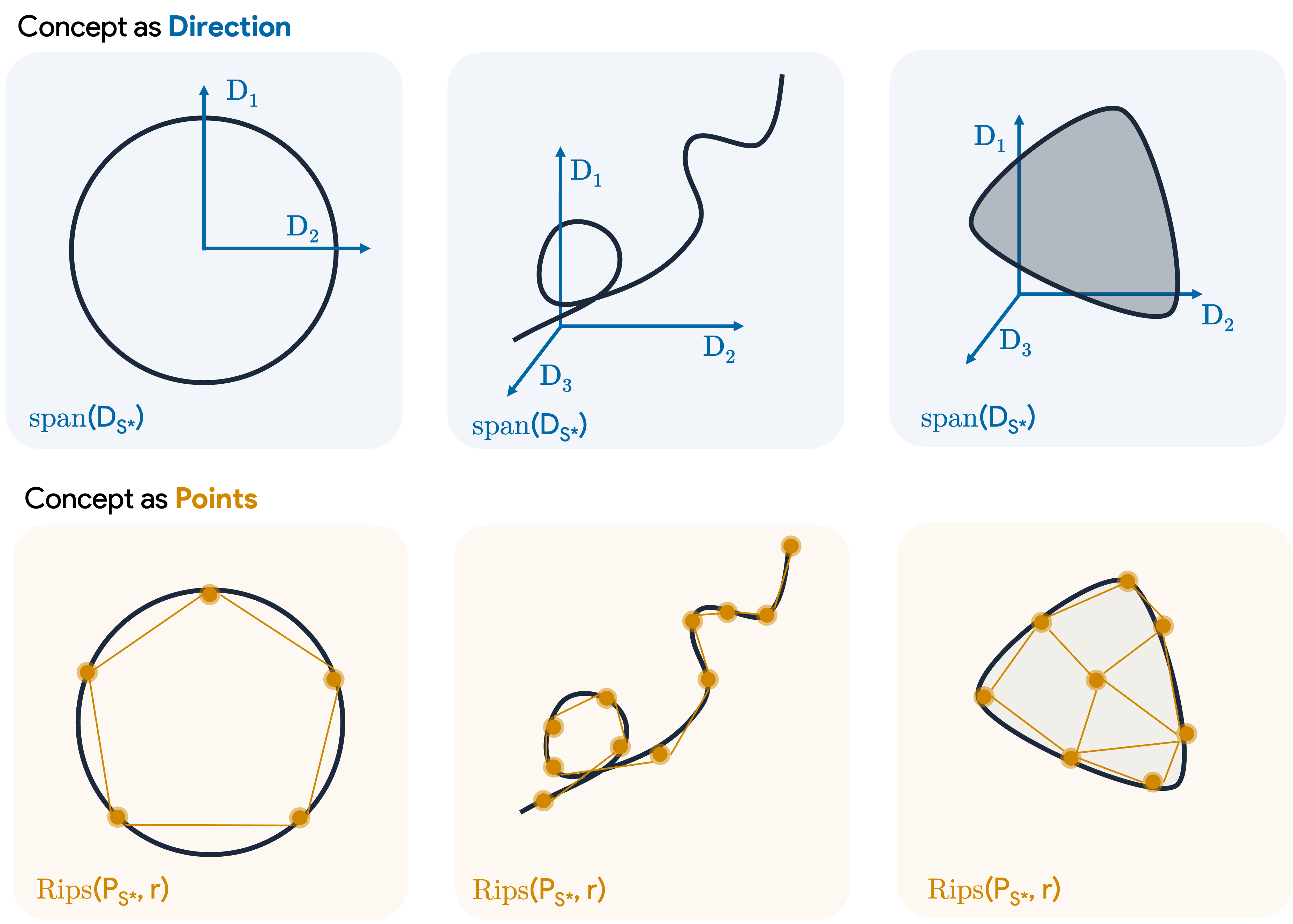}
    \caption{\textbf{The Geometric Duality of Sparse Concepts (Definition \ref{def:concept_duality}).} (Left) \textbf{Concept as Direction}: SAEs capture manifolds extrinsically by finding a fixed atom group whose linear span contains the manifold. \textbf{Concept as Points}: SAEs capture manifolds intrinsically by sampling landmarks that form a Vietoris–Rips complex homotopy equivalent to the original manifold shape. (Right)}
    \label{fig:concept_duality}
\end{figure}

It is now well-established that current approaches for concept recovery are fundamentally instances of sparse dictionary learning. In this work, we mainly studied one particular cases of Dictionary learning that make an implicit assumptions about the geometry of concepts~\citep{hindupur2025projecting}: that concepts are directions. This implicit definition is important for the discussion here, as it strictly dictates the topology of the reconstructed space and, consequently, will define what it \textit{means} mathematically to successfully recover a concept manifold. 
Briefly, we could cluster the recovery methods in 2 groups and use the implicit definition of concept as foundational distinction:

\begin{definition}[Geometric Duality of Sparse Concepts]
\label{def:concept_duality}
Given an activation $\bm{x} \in \mathcal{A}$, sparse dictionary learning extracts a latent representation $\bm{z} \in \mathbb{R}^c$ via a dictionary $\bm{D} \in \mathbb{R}^{c \times d}$ by solving the following optimization:
\begin{equation}
\label{eq:dictionarylearning}
\argmin_{\substack{\bm{z}\in\mathcal{Z}~,~ \bm{D}\in\Omega}} \|\bm{x} - \bm{z}\bm{D}\|_2^2 
+ \lambda \mathcal{R}(\bm{z}) 
\quad \text{s.t.} \quad
\begin{cases} 
\mathcal{Z} = \mathbb{R}^c_+, \quad \Omega = \mathcal{B}^{c \times d} & \text{Concepts as \textbf{Directions}} \\[6pt] 
\mathcal{Z} = \Delta^{c-1}, \quad \Omega = \mathbb{R}^{c \times d} & \text{Concepts as \textbf{Points}} 
\end{cases}
\end{equation}
where $\mathcal{R}(\bm{z})$ is a sparsity-promoting regularizer (e.g., restricting $\|\bm{z}\|_0 \leq k$). 
\end{definition}

AS a recall, the localized reconstructions $\hat{\bm{x}} = \bm{z}\bm{D}$ lie \i{i} under the \textit{directional} paradigm, in a sparse non-negative span (a cone), and we note that contrary to classical SAE, \i{ii} under the \textit{point} paradigm, reconstruction lie strictly within a sparse convex hull (a bounded polytope).
This work already studied what it means to recover a manifold for the first case, we will now study what it means to recover a manifold in the second case: concepts as points.

\subsection{Simplicial recovery}
In the point paradigm, the dictionary atoms serve as
localized landmarks and reconstruction is constrained to their convex hull. 
As both are dictionary matrices are structurally in
$\mathbb{R}^{c \times d}$, we distinguish them notationally by writing $\bm{P}$ for the
dictionary rather than $\bm{D}$.
Because point-based methods reconstruct by interpolating between
landmark positions rather than combining coordinate axes, manifold
recovery is no longer a question of spanning an ambient subspace.
Instead, it requires that the active landmarks form a sufficiently
dense and faithful discrete sample of the underlying geometry.

\begin{definition}[Simplicial capture]\label{def:simplicial_capture}
A point-based SAE with landmarks
$\bm{P} = \{\bm{P}_1, \ldots, \bm{P}_c\}$ captures a manifold
$\mathcal{M}$ (with reach $\tau > 0$) at precision $\varepsilon$ if
the active landmark subset
\[
\bm{P}_{S^\star}
= \bigl\{\bm{P}_i : \exists\, \bm{x} \in \mathcal{M}
\;\text{such that}\; i \in \mathrm{supp}(\bm{z}(\bm{x}))\bigr\}
\]
satisfies $d_H(\bm{P}_{S^\star}, \mathcal{M}) \leq \varepsilon$,
where $d_H$ denotes the Hausdorff distance.
\end{definition}

The Hausdorff condition encodes two requirements: every active
landmark lies within $\varepsilon$ of $\mathcal{M}$ (the landmarks
are close to the manifold), and every point of $\mathcal{M}$ has an
active landmark within $\varepsilon$ (the landmarks cover the
manifold).
Together, $\bm{P}_{S^\star}$ forms a faithful point sample.
Under standard density and reach conditions~\citep{niyogi2008finding},
such a sample is sufficient for the Vietoris--Rips complex built from
$\bm{P}_{S^\star}$ to recover the topology of $\mathcal{M}$,
including its connected components, loops, and higher-order cycles.
When landmarks belonging to multiple manifolds
$\mathcal{M}_1, \ldots, \mathcal{M}_m$ coexist in $\bm{P}$, the
landmark neighborhood graph provides a natural tool for separating
them.
One defines a graph $\mathcal{G}_r(\bm{P})$ with an edge between
$\bm{P}_i$ and $\bm{P}_j$ whenever
$\|\bm{P}_i - \bm{P}_j\| \leq r$; this is the 1-skeleton of
$\mathrm{Rips}(\bm{P}, r)$.
If the manifolds are well separated (inter-manifold distance
$\delta \gg 2\varepsilon$), the per-manifold landmark subsets are
exactly the connected components of $\mathcal{G}_r$ for
$r \in (2\varepsilon,\, \delta - 2\varepsilon)$.
Manifold discovery in this setting reduces to connected component
extraction, or spectral clustering when the separation is less clean.

\paragraph{Factor manifolds versus joint geometry.}
The analysis above assumes that each landmark can be assigned to a
single manifold. Under superposition, this assumption breaks down,
and point-based SAEs face a fundamental obstruction to factorwise
recovery. Consider two observations $\bm{x} = \bm{m}_1 + \bm{m}_2$
and $\bm{x}' = \bm{m}_1' + \bm{m}_2'$, where $\bm{m}_1, \bm{m}_1'
\in \mathcal{M}_1$ and $\bm{m}_2, \bm{m}_2' \in \mathcal{M}_2$.
A point-based SAE reconstructs via convex combinations of landmarks:
$\hat{\bm{x}} = \sum_j z_j \bm{P}_j$ with $z_j \geq 0$ and
$\sum_j z_j = 1$. If two landmarks $\bm{P}_a \approx \bm{x}$ and
$\bm{P}_b \approx \bm{x}'$ are both active, their midpoint
$\tfrac{1}{2}\bm{P}_a + \tfrac{1}{2}\bm{P}_b \approx
\tfrac{1}{2}(\bm{m}_1 + \bm{m}_1') + \tfrac{1}{2}(\bm{m}_2 +
\bm{m}_2')$ is a valid reconstruction. But this point does not
correspond to any observation on the data manifold:
$\tfrac{1}{2}\bm{m}_1 + \tfrac{1}{2}\bm{m}_1'$ is generically not a
point on $\mathcal{M}_1$ (it is a chord, not an arc), and likewise
for $\mathcal{M}_2$. The convex hull of landmarks that tile the joint
manifold $\mathcal{M}_1 + \mathcal{M}_2$ is therefore fundamentally
different from the Minkowski sum $\text{conv}(\mathcal{M}_1) +
\text{conv}(\mathcal{M}_2)$ that would be needed for factorwise
decomposition. In other words, the simplex constraint couples all
factors: one cannot isolate the contribution of $\mathcal{M}_1$ by
selecting a subset of landmarks, because every landmark encodes a
specific joint configuration of all co-occurring concepts. \textbf{The
landmarks tile the joint data manifold as a single object rather than
decomposing it into the separate factor manifolds that generated it}.

\begin{lemma}[Point-based landmarks do not approximate factor manifolds]
\label{lem:point_no_factor}
Let $\mathcal{M}_1 \subset V_1$ and $\mathcal{M}_2 \subset V_2$ be compact 
manifolds contained in orthogonal linear subspaces $V_1, V_2 \subset \mathbb{R}^d$, 
with $\bm 0 \notin \mathcal{M}_2$. Let $\bm{P} = \{\bm{P}_1, \ldots, \bm{P}_c\} 
\subset \mathcal{M}_1 + \mathcal{M}_2$ achieve simplicial capture of the joint 
manifold at precision $\varepsilon$. Then for every landmark $\bm{P}_j$,
\begin{equation}
    d(\bm{P}_j,\, \mathcal{M}_1)
    \;\geq\;
    \inf_{\bm{m}_2 \in \mathcal{M}_2} \|\bm{m}_2\| - \varepsilon.
\end{equation}
In particular, if $\inf_{\bm{m}_2} \|\bm{m}_2\| = \delta > 0$, then every 
landmark is at distance at least $\delta - \varepsilon$ from $\mathcal{M}_1$.
\end{lemma}

\begin{proof}
By simplicial capture, there exist $\bm{m}_1 \in \mathcal{M}_1$ and 
$\bm{m}_2 \in \mathcal{M}_2$ with $\bm{P}_j = \bm{m}_1 + \bm{m}_2 + \bm{r}$ 
where $\|\bm{r}\| \leq \varepsilon$. Let $\Pi_{V_2}$ denote orthogonal projection 
onto $V_2$. For any $\bm{m}_1' \in \mathcal{M}_1 \subset V_1$, the orthogonality 
$V_1 \perp V_2$ gives $\Pi_{V_2}(\bm{m}_1) = \Pi_{V_2}(\bm{m}_1') = \bm{0}$, hence
\begin{equation}
    \|\bm{P}_j - \bm{m}_1'\| 
    \;\geq\; \|\Pi_{V_2}(\bm{P}_j - \bm{m}_1')\|
    \;=\; \|\bm{m}_2 + \Pi_{V_2}(\bm{r})\|
    \;\geq\; \|\bm{m}_2\| - \varepsilon.
\end{equation}
Taking the infimum over $\bm{m}_1' \in \mathcal{M}_1$ on the left and over 
$\bm{m}_2 \in \mathcal{M}_2$ on the right yields the result.
\end{proof}

\begin{figure}[t]
    \centering
    \includegraphics[width=0.78\linewidth]{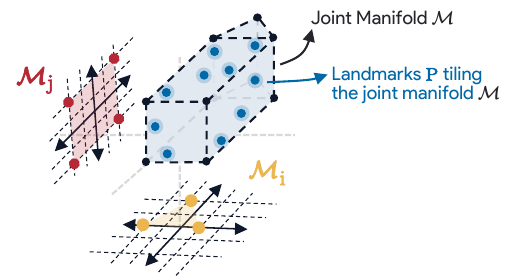}
    \caption{\textbf{Why simplicial capture cannot factor an additive mixture of manifolds.}
    A point-based dictionary tiling the \emph{joint} manifold
    $\mathcal{M} = \mathcal{M}_i + \mathcal{M}_j$ (right) does not
    induce tilings of the individual factors $\mathcal{M}_i, \mathcal{M}_j$
    (left). Each landmark $\bm{P}_k \in \mathcal{M}$ encodes one
    specific joint configuration $(\bm{m}_i, \bm{m}_j)$ of co-active
    factors, and convex combinations of landmarks reach points of
    $\mathcal{M}$ that have no preimage in any single factor. As a
    consequence (Lemma~\ref{lem:point_no_factor}), the landmarks
    cannot lie close to either $\mathcal{M}_i$ or $\mathcal{M}_j$:
    they tile a different geometric object than the factors that
    generated the data.}
    \label{fig:simplicial_obstruction}
\end{figure}

\paragraph{Point-based dictionary learning tiles the joint manifold.}
This observation effectively closes the door, for the purposes of this
paper, on point-based SAEs as a model of \emph{factor} manifold
recovery under superposition. Direction-based SAEs avoid this
obstruction because their reconstructions are additive:
$\hat{\bm{x}} = \sum_j z_j \bm{D}_j$ with no constraint coupling the
coefficients, so the partial reconstruction using only atoms aligned
with $\mathcal{M}_1$ isolates that factor's contribution regardless of
which other factors are simultaneously active. Recovering individual
factor manifolds from point-based SAEs would require replacing the
single simplex constraint with a compositional construction such as a
Minkowski sum~\citep{fel2025into} or a blockwise simplex in which
separate groups of coefficients are independently constrained to
different factors. Such extensions are interesting directions for future
work, but they fall outside the scope of the present paper. Since our
goal is to recover the geometry of individual factors, we focus in the
remainder on direction-based SAEs, for which additive structure makes
factorwise analysis well-posed.

\clearpage
\section{Conditions of Subspace capture}
\label{app:theory:directions}
Before starting, we absorb the affine offset into the SAE bias and work with 
centered points $\tilde{\bm{x}} = \bm{V}\bm{\alpha}$. The hypothesis $\mu < 1/(2k-1)$ 
is the coherence-based Exact Recovery Condition (ERC) of~\citet{tropp2004greedy}, 
which guarantees that for any signal supported on $\substar$, both Orthogonal 
Matching Pursuit and Basis Pursuit recover the correct support. Moreover, by 
Lemma~2.3 of~\citet{tropp2004greedy}, the squared singular values of 
$\bm{D}_{\substar}$ exceed $1-(k{-}1)\mu > 0$, so $\bm{D}_{\substar}$ has full 
row rank with $\|\bm{D}_{\substar}^{+}\|_2 \leq (1-(k{-}1)\mu)^{-1/2}$.

We consider the following idealized encoding setting: the dictionary $\bm{D}$ 
is obtained from SAE training, and representations are then computed by an 
Orthogonal Matching Pursuit (OMP) procedure over the learned dictionary. This 
separates the quality of the dictionary from the behavior of any particular 
feedforward encoder, and allows us to leverage classical sparse recovery 
guarantees~\citep{tropp2004greedy, donoho2003optimally}.

We restate the theorem for convenience.

\begin{theorem}[Subspace recovery]
Let $\M$ lie in a $k$-dimensional affine subspace with orthonormal basis 
$\bm{V} \in \mathbb{R}^{k \times d}$ and offset $\bm{b}_\M$. Let $\bm{D}$ be 
$\mu$-incoherent, and suppose there exists $\substar \subset [c]$ with 
$|\substar| = k$ such that $\mathrm{Im}(\bm{V}) = \mathrm{span}(\bm{D}_{\substar})$ 
and $\mu < 1/(2k-1)$. If the SAE achieves reconstruction error 
$\|\bm{x}_m - \bm{D}\bm{z}(\bm{x}_m)\| \leq \lambda$ on $\M$, then it captures 
$\M$ at precision $O(\lambda)$.
\end{theorem}

\begin{proof}
Since $\mathrm{Im}(\bm{V}) = \mathrm{span}(\bm{D}_{\substar})$, every centered 
point $\tilde{\bm{x}}_m \in \M$ admits a unique representation 
$\tilde{\bm{x}}_m = \bm{D}_{\substar}\bm{c}^\star$ for some 
$\bm{c}^\star \in \mathbb{R}^k$. Let $\bm{z}$ denote the SAE code with support 
$S = \mathrm{supp}(\bm{z})$, residual $\bm{r} = \tilde{\bm{x}}_m - \bm{D}\bm{z}$, 
$\|\bm{r}\| \leq \lambda$, and write $\bar{S} = S \setminus \substar$. Extending 
$\bm{c}^\star$ by zeros to a vector in $\mathbb{R}^c$, define 
$\bm{\delta} = \bm{z} - \bm{c}^\star$. Then
\begin{equation}
    \|\bm{D}\bm{\delta}\| 
    \;=\; \|\bm{D}\bm{z} - \bm{D}_{\substar}\bm{c}^\star\| 
    \;=\; \|\bm{D}\bm{z} - \tilde{\bm{x}}_m\| 
    \;\leq\; \lambda.
\end{equation}
Under the ERC, the noise-robust null-space property of~\citet{tropp2006just} 
(Theorem~14) yields
\begin{equation}\label{eq:nsp_bound}
    \|\bm{\delta}_{\bar{S}}\|_1 
    \;\leq\; \frac{C(\mu, k)}{1 - (2k-1)\mu}\,\|\bm{D}\bm{\delta}\|_2 
    \;=\; O(\lambda),
\end{equation}
where $C(\mu, k)$ depends only on $\mu$ and $k$. Since $\bm{c}^\star$ is 
supported on $\substar$, $\bm{\delta}_{\bar{S}} = \bm{z}_{\bar{S}}$, hence 
$\|\bm{z}_{\bar{S}}\|_1 = O(\lambda)$.

The $\substar$-restricted reconstruction error then satisfies
\begin{align}
    \Bigl\|\tilde{\bm{x}}_m - \sum_{i \in \substar} z_i(\tilde{\bm{x}}_m)\bm{D}_i\Bigr\|
    \;&=\; \|\tilde{\bm{x}}_m - \bm{D}_{\substar}\bm{z}_{\substar}\| \nonumber\\
    \;&=\; \|\bm{D}_{\bar{S}}\bm{z}_{\bar{S}} + \bm{r}\| \nonumber\\
    \;&\leq\; \|\bm{z}_{\bar{S}}\|_1 + \lambda 
    \;=\; O(\lambda),
\end{align}
where the first inequality uses unit-norm atoms and the triangle inequality. 
Thus the SAE captures $\M$ at precision $O(\lambda)$ in the sense of 
Defn.~\ref{def:subspace_capture}.
\end{proof}

\clearpage
\section{Synthetic Experiment Details}
\label{app:synthetic}

\begin{figure}[h]
    \centering
    \includegraphics[width=1.0\linewidth]{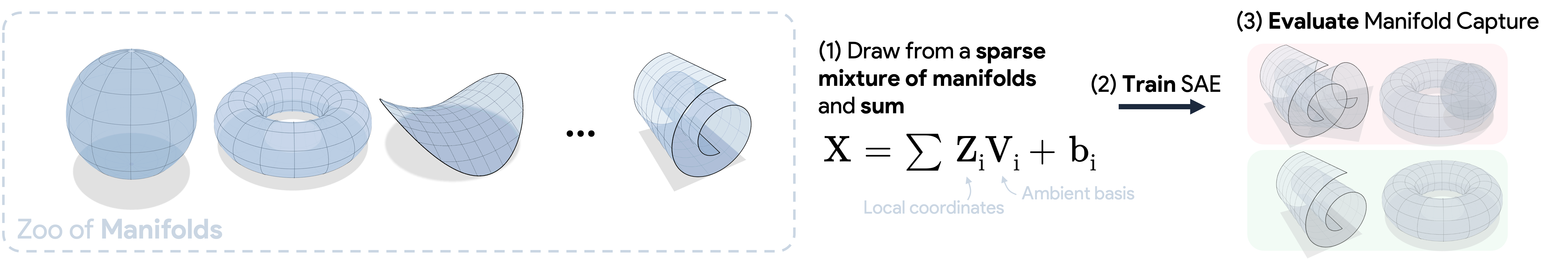}
    \caption{\textbf{Synthetic Evaluation Pipeline.} We construct a controlled benchmark for manifold recovery by sparse autoencoders. \textbf{(1)} We define a zoo of manifolds (spheres, tori, Möbius strips, etc.) and generate data points by sampling from a sparse mixture: each observation is formed as $\bm{X} = \sum_i \bm{Z}_i \bm{U}_i$, where $\bm{Z}_i$ are local coordinates on the $i$-th manifold and $\bm{U}_i$ are ambient basis matrices embedding each manifold into high-dimensional space. \textbf{(2)} An SAE is trained on the resulting superposed activations. \textbf{(3)} We evaluate whether the SAE recovers the individual manifolds from the mixture, assessing both subspace capture (for direction-based SAEs) and simplicial capture (for point-based SAEs) as defined in Sections~\ref{def:subspace_capture} and~\ref{def:simplicial_capture}.}
\label{fig:toy_dataset_presentation}
\end{figure}

\begin{figure}[h]
    \centering
    \includegraphics[width=0.99\linewidth]{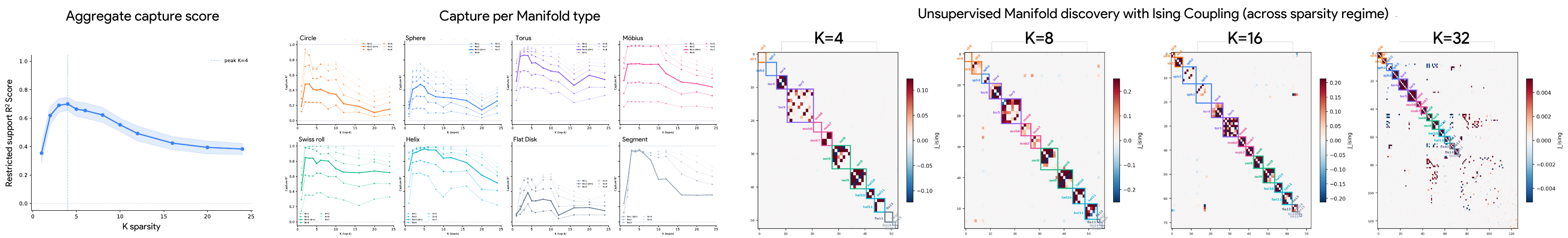}
    \caption{Ising coupling matrix $\bm{J}_{ij}$ recovers latent manifold structure across sparsity regimes. Red = positive coupling ; blue = mutual exclusion ($J<0$). At low $K$ (tiling), atoms are shared across manifolds and block structure is weak. At intermediate $K$ ($K \approx 8$–$16$), clean block-diagonal structure emerges. At high $K$ (dilution), atoms over-tile individual manifolds, fragmenting blocks.
}
    \label{fig:ising_appendix}
\end{figure}

\paragraph{Manifold zoo.}
Table~\ref{tab:manifold_zoo} summarizes the eight manifold types used in the synthetic benchmark. For each type, we list the intrinsic dimension $d_i$ (the number of free parameters), the embedding dimension $k_i$ (the dimension of the ambient subspace containing the manifold, which determines the number of atoms needed for subspace capture), the parametric embedding $\gamma_i$, and the parameter ranges used across variants.

\begin{table}[h]
\centering
\footnotesize
\caption{Manifold zoo used in the synthetic benchmark.}
\label{tab:manifold_zoo}
\setlength{\tabcolsep}{2pt}
\begin{tabular}{lccll}
\toprule
\textbf{Type} & $d_i$ & $k_i$ & \textbf{Embedding} $\gamma_i(\theta) \in \mathbb{R}^{k_i}$ & \textbf{Variant parameters} \\
\midrule
Circle & 1 & 2 & $(r\cos\theta,\; r\sin\theta)$ & $r \in \{0.5, 0.75, 1.0, 1.5, 2.0, 3.0\}$ \\
Sphere & 2 & 3 & $(r\sin\phi\cos\theta,\; r\sin\phi\sin\theta,\; r\cos\phi)$ & $r \in \{0.5, 0.75, 1.0, 1.5, 2.0, 3.0\}$ \\
Torus & 2 & 4 & Clifford: $((R{+}r\cos\phi)\cos\theta,\; \ldots,\; r\sin\phi)$ & $(R,r) \in \{(2,0.5), (2,1), (3,1), \ldots\}$ \\
M\"obius & 2 & 3 & $((1{+}t\cos\tfrac{\phi}{2})\cos\phi,\; \ldots,\; t\sin\tfrac{\phi}{2})$ & $w \in \{0.2, 0.3, 0.5, 0.7, 1.0, 1.5\}$ \\
Swiss roll & 2 & 3 & $(\theta\cos\theta,\; h,\; \theta\sin\theta)$ & $(\theta_{\max}, h_{\max}) \in \{(2\pi, 1.5), \ldots, (4.5\pi, 6)\}$ \\
Helix & 1 & 3 & $(r\cos\theta,\; r\sin\theta,\; \alpha\theta)$ & $\alpha \in \{0.1, 0.2, 0.3, 0.4, 0.5, 0.6\}$, $r{=}1$, 3 turns \\
Flat disk & 2 & 2 & $(r\cos\theta,\; r\sin\theta)$ with $r \sim \sqrt{\mathcal{U}(0,1)}$ & $R \in \{0.5, 0.75, 1.0, 1.5, 2.0, 3.0\}$ \\
Segment & 1 & 1 & $(t)$ & length $\in \{0.5, 0.75, 1.0, 1.5, 2.0, 3.0\}$ \\
\bottomrule
\end{tabular}
\end{table}

The distinction between $d_i$ and $k_i$ is important throughout the paper. The intrinsic dimension $d_i$ governs the manifold's degrees of freedom and determines the expected number of localized detectors in the tiling regime. The embedding dimension $k_i$ determines the number of atoms required for subspace capture (Definition~\ref{def:subspace_capture}): a circle is parameterized by a single angle ($d_i = 1$) but its embedding $(\cos\theta, \sin\theta)$ lives in a 2-dimensional subspace ($k_i = 2$), so two atoms are needed to span it. Similarly, the torus is intrinsically 2-dimensional but requires a 4-dimensional Clifford embedding to faithfully represent its topology.

\paragraph{Normalization.}
A critical design choice is ensuring that all manifold instances contribute equally to the reconstruction loss. Without normalization, manifold types with large embeddings (e.g., the Swiss roll, whose coordinates scale as $\theta \sim 3\pi$ to $4.5\pi$) would dominate the SAE's capacity, while small-norm manifolds (e.g., a circle with $r = 0.5$) would be treated as noise. We address this by centering and isotropically rescaling each instance at construction time. Concretely, for each manifold instance $i$ with parameters $\rho_i$, we draw a calibration sample of 50{,}000 points from the raw embedding $\gamma_i$, compute the sample mean $\bm{\mu}_i$ and the RMS norm of the centered samples $\sigma_i = \sqrt{\mathbb{E}[\|\gamma_i(\theta) - \bm{\mu}_i\|^2]}$, and define the normalized embedding as
\begin{equation}
    \tilde{\gamma}_i(\theta) = \frac{\gamma_i(\theta) - \bm{\mu}_i}{\sigma_i}.
\end{equation}
This transformation is an isotropic rescaling composed with a translation: it preserves all angles, relative distances, curvature ratios, and topological structure. After normalization, every instance has RMS norm exactly 1 in local coordinates, regardless of manifold type or variant parameters.

\paragraph{Ambient embedding.}
For each of the 48 manifold instances (8 types $\times$ 6 variants), we draw a random orthonormal matrix $\bm{V}_i \in \mathbb{R}^{k_i \times d}$ by sampling a $d \times k_i$ Gaussian matrix and taking the $\bm{Q}$ factor of its QR decomposition (transposed to obtain orthonormal rows). This ensures that $\|\bm{z}\bm{V}_i\|_2 = \|\bm{z}\|_2$ for all $\bm{z}$: the ambient embedding is norm-preserving. The bias vectors $\bm{b}_i$ are set to zero throughout (i.e., $\sigma_{\text{bias}} = 0$).

\paragraph{Sparse mixture sampling.}
Observations follow the generative model
\begin{equation}
    \bm{x} = \sum_{i \in S} \tilde{\gamma}_i(\theta_i)\,\bm{V}_i + \bm{\epsilon}, \qquad |S| = L_0,
\end{equation}
where the active set $S$ is drawn uniformly at random (without replacement) from the 48 instances, intrinsic coordinates $\theta_i$ are sampled uniformly on each manifold, and $\bm{\epsilon} \sim \mathcal{N}(\bm{0}, \sigma_\epsilon^2 \bm{I}_d)$ with $\sigma_\epsilon = 10^{-5}$. The noise level is deliberately kept small so that reconstruction quality reflects the SAE's geometric organization rather than denoising ability.

We generate $N = 2{,}000{,}000$ training samples. The evaluation set consists of $1{,}000{,}000$ samples generated at $L_0 = 4$ with a separate random seed, along with the corresponding per-manifold contributions $\bm{m}_i = \tilde{\gamma}_i(\theta_i)\bm{V}_i$ and active masks. This separation ensures that evaluation measures capture on in-distribution superposition with known ground truth.

\paragraph{SAE Training}
We use TopK sparse autoencoders throughout the synthetic experiments. The encoder is a linear map $\bm{W}_{\text{enc}} \in \mathbb{R}^{c \times d}$ followed by TopK selection (retaining only the $k$ largest activations and zeroing the rest). The decoder is a linear map $\bm{W}_{\text{dec}} \in \mathbb{R}^{d \times c}$ with unit-norm columns, applied to the sparse code to produce the reconstruction $\hat{\bm{x}} = \bm{W}_{\text{dec}}\,\text{TopK}(\bm{W}_{\text{enc}}\,\bm{x})$. The dictionary size is $c = 512$ throughout, yielding an expansion factor of $c/d = 4$ relative to the ambient dimension $d = 128$.

We train separate SAEs for each sparsity budget $k \in \{3, 4, 6, 8, 10, 14, 16, 20, 25\}$. This range is chosen to span all three theoretical regimes. 

All SAEs are trained with Adam (learning rate $3 \times 10^{-3}$, no weight decay) for 10 epochs with batch size 1{,}024. The loss function combines $\ell_1$ reconstruction error with a dead-neuron reanimation term.
An atom is considered dead if it has zero activation for every sample in the current batch. The reanimation term encourages dead atoms to develop nonzero pre-activations, preventing capacity waste.

\paragraph{Restricted $R^2$ (subspace capture score).}
The primary metric tests Definition~\ref{def:subspace_capture} directly. For a given SAE trained at sparsity $k$, we proceed as follows:
\begin{enumerate}
    \item Encode the full evaluation set $\{\bm{x}^{(j)}\}$ through the SAE to obtain codes $\{\bm{z}^{(j)}\}$.
    \item For each manifold instance $i$, select the rows where $i$ is active (using the ground-truth active masks) to obtain the manifold-specific codes $\bm{Z}_i \in \mathbb{R}^{n_i \times c}$ and the corresponding true contributions $\bm{M}_i \in \mathbb{R}^{n_i \times d}$.
    \item Greedily select $n$ atoms by iteratively choosing the decoder direction $\bm{d}_j$ that explains the most residual variance of $\bm{M}_i$. At each step, the selected atom's projection is removed from the residual before selecting the next.
    \item Mask the codes to retain only the $n$ selected atoms: $\bm{Z}_i^{(n)} = \bm{Z}_i \odot \bm{e}_{\text{selected}}$, where $\bm{e}_{\text{selected}}$ is a binary mask.
    \item Decode: $\hat{\bm{M}}_i^{(n)} = \bm{Z}_i^{(n)}\,\bm{W}_{\text{dec}}^{\top}$.
    \item Compute the restricted $R^2$:
\end{enumerate}
\begin{equation}
    R^2(i, k, n) = 1 - \frac{\sum_j \|\bm{m}_i^{(j)} - \hat{\bm{m}}_i^{(j,n)}\|^2}{\sum_j \|\bm{m}_i^{(j)} - \bar{\bm{m}}_i\|^2},
\end{equation}
where $\bar{\bm{m}}_i$ is the mean of the true contributions. An $R^2$ near 1 at $n = k_i$ indicates compact subspace capture. We report $R^2$ for $n$ ranging from $\max(1, k_i - 2)$ to $k_i + 2$ to visualize how capture improves around the embedding dimension.

Note that the greedy selection operates on the decoder directions of the trained SAE, not on the codes. This is important: we are asking whether $n$ decoder directions span the manifold's ambient subspace, using the codes the SAE actually produces on in-distribution (superposed) inputs.

\paragraph{Support size.}
For each manifold instance $i$ and SAE sparsity $k$, the support size $|S_{\mathcal{M}}|$ counts the number of unique dictionary atoms that fire on at least 10\% of the manifold's evaluation points. To avoid counting near-zero activations (e.g., from ReLU tails or numerical noise), we apply a per-atom magnitude threshold: for each atom $j$, we compute the 10th percentile of its nonzero activations and discard activations below this threshold. Atoms must additionally fire on at least 30 points (an absolute floor) to be counted. This filtering ensures that the support size reflects genuinely active atoms rather than numerical artifacts.

\paragraph{Receptive field spread.}
For each atom $j$ in the support of manifold $i$, we gather all manifold points where $j$ fires (after the robustness filtering described above) and compute the mean pairwise Euclidean distance among those points in ambient space. This quantity measures how broadly the atom's receptive field covers the manifold. We then take the median across all atoms in the support and normalize by the manifold's own mean pairwise distance (computed from a subsample of up to 2{,}000 points), yielding a dimensionless quantity between 0 (maximally localized: each atom fires on a tight cluster) and 1 (maximally global: each atom fires uniformly across the entire manifold).

\paragraph{Ising coupling inference.}
To recover manifold structure from SAE codes without supervision, we binarize the codes ($s_j = \text{sign}(z_j)$) and fit a pairwise Ising model
\begin{equation}
    p(\bm{s}) \propto \exp\left(\sum_{i < j} J_{ij}\,s_i\,s_j + \sum_i h_i\,s_i\right)
\end{equation}
using pseudo-likelihood maximization (PLM) with L-BFGS optimization. We enforce symmetry by setting $\bm{J} = (\bm{W} + \bm{W}^{\top})/2$ during optimization rather than as a post-hoc correction. Regularization strength is selected via the extended Bayesian information criterion (EBIC) with $\gamma = 0.5$, following the IsingFit procedure of \cite{van2014new}. The fields $h_i$ absorb marginal firing rates, so universally active atoms have large $|h_i|$ but small $|J_{ij}|$, and indirectly correlated atoms are factored out by construction. We then apply Louvain community detection to $|\bm{J}|$ to partition atoms into candidate manifold groups, and validate each group by checking for a sharp PCA spectral gap in its code vectors (indicating low-dimensional structure consistent with a manifold).

\clearpage
\section{Recovering Manifold Structure via Ising Model}
\label{app:ising}

Recovering the manifold partition from SAE codes can be cast as a problem of graphical model selection over dictionary atoms.
The goal is to infer which atoms are structurally related (jointly tile or span the same manifold) and which are independent, using only the observed pattern of activations.
This section formalizes the connection, states the conditions under which recovery succeeds, and motivates the two-stage pipeline used in the main text.

\paragraph{From Covariance to Conditional Independence}

A natural first approach to grouping atoms is to examine their pairwise covariance.
However, covariance conflates direct and indirect statistical dependencies: two atoms may be correlated not because they tile the same manifold, but because a third atom (or a latent variable such as topic) mediates their interaction~\citep{dempster1972covariance}.
The classical remedy is to examine the precision matrix $\bm{\Omega} = \bm{\Sigma}^{-1}$ instead.
For jointly Gaussian random variables, the precision matrix encodes the conditional independence structure exactly: $\Omega_{ab} = 0$ if and only if $z_a \perp\!\!\!\perp z_b \mid \bm{z}_{\setminus\{a,b\}}$~\citep{lauritzen1996graphical}.
Estimating $\bm{\Omega}$ from data is the subject of Gaussian graphical model selection, with well-studied algorithms such as the graphical lasso~\citep{friedman2008sparse} and neighborhood selection~\citep{meinshausen2006high}.

SAE codes, however, are not Gaussian.
They are sparse (most entries zero), non-negative (due to ReLU or TopK), and their support is constrained (TopK enforces $\|\bm{z}\|_0 = k$).
Applying Gaussian graphical model selection to such data would yield inconsistent estimates of the conditional independence structure, since the Gaussian likelihood is misspecified.
We therefore require a graphical model adapted to the discrete, binary nature of SAE activation patterns.

\paragraph{The Ising Model as a Binary Graphical Model}

A principled alternative is to binarize the codes, setting $s_a = \text{sign}(z_a)$, and model their joint distribution with a pairwise exponential family.
For binary random variables, the maximum-entropy distribution consistent with observed first and second moments $\mathbb{E}(s_a)$ and $\mathbb{E}(s_a s_b)$ is the \emph{Ising model}~\citep{jaynes1957information, wainwright2008graphical}:
\begin{equation}
\label{eq:ising_app}
p(\bm{s}) \;=\; \frac{1}{Z(\bm{J}, \bm{h})} \exp\!\Bigl(\sum_{a < b} J_{ab}\, s_a s_b \;+\; \sum_a h_a\, s_a\Bigr),
\end{equation}
where the couplings $\bm{J} \in \mathbb{R}^{c \times c}$ parameterize pairwise interactions, the fields $\bm{h} \in \mathbb{R}^c$ capture marginal activation rates, and $Z(\bm{J}, \bm{h})$ is the partition function.
This is a well-studied model in statistical physics~\citep{ising1925beitrag}, computational neuroscience~\citep{schneidman2006weak, cocco2009neuronal}, and machine learning~\citep{wainwright2008graphical}.
Its use for modeling neural co-activation statistics was pioneered by Schneidman et al.~\citep{schneidman2006weak}, who showed that pairwise interactions account for the vast majority of multi-neuron correlations in biological neural populations.

The key property connecting the Ising model to manifold recovery is the following classical result.

\begin{proposition}[Pairwise Markov property of the Ising model; \citet{besag1974spatial}]
\label{prop:hc}
Let $p(\bm{s}) > 0$ for all $\bm{s} \in \{0,1\}^c$.
Then the distribution $p$ factorizes as in Eq.~\eqref{eq:ising_app} if and only if
\begin{equation}
J_{ab} = 0 \quad \iff \quad s_a \perp\!\!\!\perp s_b \mid \bm{s}_{\setminus\{a,b\}}.
\end{equation}
That is, the support of $\bm{J}$ encodes exactly the conditional independence graph of the distribution.
\end{proposition}

This result, a special case of the Hammersley-Clifford theorem~\citep{besag1974spatial, lauritzen1996graphical}, establishes that the Ising couplings play the same role for binary variables that the precision matrix plays for Gaussian variables: $J_{ab} = 0$ if and only if atoms $a$ and $b$ are conditionally independent given all others.
Fitting the Ising to binarized SAE codes is therefore the construction prescribed by the theory of undirected graphical models for binary data~\citep{wainwright2008graphical}.

\section{A Geometric and Statistical View of Tiling}
\label{app:ising_regimes}

The empirical results of Sec.~\ref{sec:characterizing} establish that trained SAEs 
rarely operate in the capture regime: rather than reusing a compact group of atoms 
across an entire manifold, they fragment $\M$ into many partially overlapping 
receptive fields. This fragmented allocation is what we call \emph{tiling}, and it 
admits two qualitatively different geometric forms. Before giving formal 
definitions through Ising couplings, we describe the underlying picture.

\paragraph{Tiling by individual atoms: Shattering}
The cleanest form of tiling allocates a single atom to each region of $\M$. Each 
atom $\bm{d}_i$ activates on a localized patch $\mathcal{P}_i \subseteq \M$, and the patches 
$\{\mathcal{P}_i\}_{i \in G_\M}$ partition the manifold with little overlap. On any input 
$\x_m \in \M$, exactly one (or few compared to the ambiant space) of the group's atoms fires; moving 
along $\M$ corresponds to handing off activity from one atom to the next, like 
the firing of place cells as an animal traverses an environment~\citep{o1971hippocampus}.
The total number of atoms used to represent $\M$ is therefore proportional to the 
volume of $\M$ in atom-units, not to its ambient dimension $k_\M$: a 
larger manifold simply needs more tiles.

\paragraph{Tiling by group of atoms: Dilution}
Dilution is the same idea applied at the level of groups rather than individuals. 
Instead of one atom firing per region, a small subset $G_{\x} \subseteq G_\M$ fires 
on each $\x_m \in \M$, and the subsets $G_{\x}$ vary smoothly (and overlap 
heavily) as $\x_m$ moves along $\M$. No single $G_{\x}$ is small enough to 
satisfy capture at the manifold scale, but each is locally redundant: many atoms 
encode the same region in parallel. The total support $G_\M$ is then 
much larger than the ambient dimension $k_\M$, often by an order of 
magnitude, because every region is covered by several atoms rather than one.

\paragraph{Why the distinction matters.}
Both regimes preserve $\M$'s geometry implicitly through joint activity, and both 
are consistent with low reconstruction error. Geometrically, however, they 
correspond to opposite assumptions about what an atom is for: shattering treats 
atoms as landmarks (point-like detectors with sharp receptive fields), 
while dilution treats them as local redundant basis (overlapping linear 
contributions whose sum reconstructs the local geometry). Shattering is closer 
in spirit to the point paradigm of App.~\ref{app:duality}; dilution is closer to 
the directional paradigm but with the wrong sparsity budget. Distinguishing them 
is essential because they will exhibit different statistical signature and thus recovering the manifold will require different strategies.

\paragraph{Statistical signatures.}
The two regimes leave distinct fingerprints in the joint activation statistics 
of $G_\M$. Under shattering, the single-atom-per-region structure means atoms 
are mutually exclusive: when $a$ fires, $b \neq a$ does not, and vice 
versa. Under dilution, atoms within $G_{\x}$ are positively coupled (they fire 
together on the same region), but atoms in disjoint $G_{\x}$ and $G_{\x'}$ 
inhibit each other. Capture, in contrast, places all of $G_\M$ inside every 
$G_{\x}$: every pair co-fires on every input. The Ising model makes these 
qualitative claims precise.

\begin{definition}[Ising signatures of capture, shattering, and dilution]
\label{def:ising_regimes}
Let $G \subseteq [c]$ be a group of atoms hypothesized to represent a manifold 
$\M$. Define the \emph{signed cohesion}
\begin{equation}
\label{eq:signed_cohesion}
\rho(G) \;=\; \frac{1}{\binom{|G|}{2}} \sum_{\substack{a, b \in G \\ a < b}} 
\operatorname{sign}(J_{ab}) \;\in\; [-1, +1].
\end{equation}
For thresholds $\tau \in (0, 1]$, $G$ is in the:
\begin{itemize}[leftmargin=*]
    \item \textbf{Capture regime} if $|G| \approx k_\M$ and $\rho(G) \geq +\tau$. 
    \hfill(all atoms co-fire)
    \item \textbf{Shattering regime} if $|G| \gg k_\M$ and $\rho(G) \leq -\tau$. 
    \hfill (atoms mutually exclude)
    \item \textbf{Dilution regime} if $|G| \gg k_\M$ and $|\rho(G)| < \tau$. 
    \hfill (mixed couplings, overlapping sub-groups)
\end{itemize}
\end{definition}

The key auxiliary quantity is $|G|$ itself: capture requires $|G|$ to be small 
(no more than the ambient dimension), while both forms of tiling require it to 
be large. Within the tiling regimes, the sign of the couplings then 
distinguishes the geometric mechanism.

\paragraph{Implications for SAE interpretability.}
Two consequences follow. First, both ordered states (capture and shattering, ferromagnetic and 
antiferromagnetic respectively) yield well-defined communities under spectral clustering of 
$|\bm{J}|$, while the disordered state (dilution) does not --- consistent with our empirical 
finding that recovering manifolds in the dilution regime requires the additional 
spectral-gap validation step of Sec.~\ref{sec:discovery}. 
Second, the framing identifies dilution as the genuinely problematic regime for interpretability: 
it preserves geometry but does so without any coherent organizational 
principle. The empirical evidence assembled in Sec.~\ref{sec:characterizing} 
restricted-$R^2$ curves that plateau well beyond $k_\M$, tuning curves that  fragment manifolds across many partially redundant features, and Ising couplings whose intra-group sign structure is mixed rather than uniformly positive or negative consistently places trained SAEs in this regime. 
We therefore read current SAEs not as failed instances of capture, but as diluted representations: they do encode manifold geometry, but distribute it across a redundant cover whose organizational logic is invisible at the level of any single feature, and only partially recoverable through 
post-hoc analysis we intended to develop in Sec.~\ref{sec:discovery}.

\end{document}